\documentclass{article}

\usepackage{microtype}
\usepackage{graphicx}
\usepackage{subcaption}
\usepackage{booktabs} 
\usepackage{multirow}

\usepackage{hyperref}

\usepackage[preprint]{icml2026}

\usepackage{amsmath}
\usepackage{amssymb}
\usepackage{mathtools}
\usepackage{amsthm}

\usepackage[capitalize,noabbrev]{cleveref}

\theoremstyle{plain}
\newtheorem{theorem}{Theorem}[section]
\newtheorem{proposition}[theorem]{Proposition}

\theoremstyle{definition}
\newtheorem{definition}[theorem]{Definition}

\theoremstyle{remark}

\theoremstyle{definition}
\newtheorem{remark-bold}[theorem]{Remark}

\usepackage{algorithm}
\usepackage{outlines}

\theoremstyle{hypothesis}

\usepackage[textsize=tiny]{todonotes}

\icmltitlerunning{Bias-Constrained Diffusion Schedules for PDE Emulations}

\begin{document}

\twocolumn[
  \icmltitle{Bias-Constrained Diffusion Schedules for PDE Emulations: \\  Reconstruction Error Minimization and Efficient Unrolled Training}



  \icmlsetsymbol{equal}{*}

  \begin{icmlauthorlist}
    \icmlauthor{Constantin Le Cleï}{yyy}
    \icmlauthor{Nils Thuerey}{yyy}
    \icmlauthor{Xiaoxiang Zhu}{yyy}
  \end{icmlauthorlist}

  \icmlaffiliation{yyy}{Technical University of Munich, Munich, Germany}

  \icmlcorrespondingauthor{Constantin Le Cleï}{constantin.le.clei@tum.de}

  \icmlkeywords{Machine Learning, ICML}

  \vskip 0.3in
]



\printAffiliationsAndNotice{}  

\begin{abstract}
    Conditional Diffusion Models are powerful surrogates for emulating complex spatiotemporal dynamics, yet they often fail to match the accuracy of deterministic neural emulators for high-precision tasks. In this work, we address two critical limitations of autoregressive PDE diffusion models: their sub-optimal single-step accuracy and the prohibitive computational cost of unrolled training. First, we characterize the relationship between the noise schedule, the reconstruction error reduction rate and the diffusion exposure bias, demonstrating that standard schedules lead to suboptimal reconstruction error. Leveraging this insight, we propose an \textit{Adaptive Noise Schedule} framework that minimizes inference reconstruction error by dynamically constraining the model's exposure bias. We further show that this optimized schedule enables a fast \textit{Proxy Unrolled Training} method to stabilize long-term rollouts without the cost of full Markov Chain sampling. Both proposed methods enable significant improvements in short-term accuracy and long-term stability over diffusion and deterministic baselines on diverse benchmarks, including forced Navier-Stokes, Kuramoto-Sivashinsky and Transonic Flow.
\end{abstract}

\section{Introduction}

Machine learning-based emulators have achieved performance comparable to traditional numerical solvers for fluid dynamics tasks, operating at a fraction of the computational cost. Among these, Diffusion Models have recently demonstrated strong performance while resolving key limitations of deterministic approaches, such as instability \cite{kohl2023benchmarking, ruhling2023dyffusion, lippe2023pde}, and over-smoothing over long simulation horizons \cite{liu2025rolling}. Beyond stability, they allow to obtain uncertainty estimates \cite{liu2024uncertainty}, can be used flexibly across tasks \cite{shysheya2024conditional} for super-resolution and forecasting, and can sample multiple consistent scenarios \cite{price2023gencast}

However, Diffusion Models typically lag behind deterministic baselines in terms of pure reconstruction accuracy, which can lead to biased estimates despite their probabilistic formulation. Furthermore, because their sampling process requires tens to hundreds of function evaluations, standard stabilization techniques, such as multi-step unrolled training, are computationally prohibitive to apply directly.

In this work, we argue that these limitations are intrinsically linked through the phenomenon of \textit{Diffusion Exposure-Bias} \cite{li2023alleviating}. This refers to the performance drop caused by the discrepancy between the ground-truth noisy targets seen during training and the estimated noisy states generated via ancestral sampling at inference.  In the case of spatio-temporal forecasting, we demonstrate that standard conditional diffusion models suffer from this bias and obtain a sub-optimal error, mainly because \textit{their noise schedules are not designed to align the rate of error reduction with the intrinsic model capacity and task difficulty}. In particular, we make the following contributions:
\begin{itemize}
    \item  We define the \textit{Reconstruction Exposure-Bias}, a special case of exposure-bias, which denotes the discrepancy between intermediate reconstruction errors during inference and training. We experimentally show that the main driver of this bias is the rate of decrease of the reconstruction error as a function of the noise level.
    \item We propose a novel \textit{Adaptive Scheduling} algorithm that treats schedule design as a final noise-level minimization problem under the constraint that the model stays within stability at every denoising step, thus jointly optimizing the Final Reconstruction Error and the Reconstruction Exposure-Bias 
    \item Connection between the exposure-biases : Reducing the diffusion exposure-bias naturally leads to a fast in-distribution proxy for the full diffusion process, only requiring a few steps to provide an accurate sample which can be used for Unrolled training, therefore reducing simulation exposure-bias.
    \item  Our proposed adaptive schedule systematically reduces first-step reconstruction error, while the fast proxy unrolled-training drastically mitigates artifacting effects, leading to an improvement multiple orders of magnitude in Fréchet Spectral Distance on Kolmogorov turbulent flow (see Table~1).
\end{itemize}

\section{Related Work}

\paragraph{Diffusion for fluid flows} Diffusion models \cite{ho2020denoising, sohl2015deep} have been shown to perform well for autoregressively generating videos \cite{ho2022video}, naturally expanding to other spatio-temporal tasks such as weather forecasting \cite{price2023gencast} or time-series \cite{shen2024multi}. In the context of fluid dynamics, there have been multiple applications : \cite{kohl2023benchmarking} demonstrated that diffusion models can predict fluid states over extended horizons while preserving both sample quality and temporal stability. For fluid-flow reconstruction, \cite{shu2023physics} leveraged DDPMs for turbulent flow super-resolution, accurately reconstructing high-fidelity fields from low-resolution inputs, while \cite{rozet2023score} relied on score-based generative models for reconstruction. For 3D turbulent flows, \cite{lienen2306zero} achieved fast spatio-temporal prediction. \cite{liu2025rolling} merged the time and diffusion axis to produce fast physically plausible samples on very long horizons.

\paragraph{Diffusion Hyperparameters} Noise schedules are a central component of diffusion models, determining the rate of information destruction and the weighting of the learning objective \cite{kingma2021variational}. The design of these schedules has evolved from heuristic linear and cosine formulations \cite{ho2020denoising,nichol2021improved} to principled parameterizations based on the Signal-to-Noise Ratio (SNR) \cite{kingma2021variational,choi2022perception} and the continuous noise level formulations of the EDM framework \cite{karras2022elucidating}. However, these schedules are predominantly optimized for perceptual fidelity, prioritizing regimes where visual features emerge. This focus is ill-suited for high-precision fluid dynamics tasks, where the objective is not only to generate physically plausible fields, but also to remain tightly correlated with a specific ground-truth trajectory. In particular, \cite{lippe2023pde} has shown that schedules that focus on very low noise-levels tend to perform well on turbulence tasks.

\paragraph{Exposure bias in diffusion models.} This phenomenon arises from the mismatch between training, where the model sees ground-truth noisy data, and sampling, where it denoises its own predictions, causing errors to accumulate along the reverse trajectory. Different method have been developed to mitigate this effect, including scaling the predicted noise \cite{ning2023elucidating}, shifting sampling timesteps \cite{li2023alleviating} or adding a regularization term in the diffusion loss \cite{daras2023consistent}. For PDE simulations however, the target posterior is highly concentrated, unlike the broad distributions typical of unconditional image generation. Consequently, this regime requires a tailored framework to address its specific constraints.

\paragraph{External Loss terms and Unrolled Training} Constraining the output of diffusion models has become an active area of research, particularly through guidance mechanisms \cite{bansal2023universal, ho2022classifier}. In the context of trajectory generation, these constraints serve a dual purpose: facilitating accurate autoregressive predictions and improving robustness to error accumulation during unrolled training. In fluid dynamics tasks, physical consistency losses are typically enforced by approximating the clean output via DDIM shortcuts \cite{bastek2024physics} or linear interpolation \cite{amoros2026guiding}. However, such approximations are insufficient for unrolled training; they lack the fidelity required to mimic the model's actual inference behavior. Consequently, the model fails to learn how to correct its own generated artifacts, potentially leading to gradient inaccuracies. Specialized strategies to stabilize autoregressive diffusion have also emerged : \cite{chen2024diffusion} condition the model on noise-corrupted histories, while \cite{huang2025self} relies on cached prior predictions rather than ground-truth. Alternative strategies for estimating the final output include Consistency Models \cite{song2023consistency, stock2025swift} or diffusion shortcuts \cite{shehata2025improved}.

\section{Schedule, Reconstruction Error, and Exposure-Bias}

\subsection{Background: Conditional Diffusion Models for Autoregressive Fluid Simulations}

For Fluid Dynamics tasks, our goal is to construct a neural operator $\mathcal{M}_\theta$ that emulates the ground-truth time evolution of a fluid. Given an initial condition $\mathbf{x}^0$, we approximate the subsequent trajectory $\{\mathbf{x}^1, \dots, \mathbf{x}^K\}$ via autoregressive estimates:
\begin{align}
    \hat{\mathbf{x}}^1 &= \mathcal{M}_\theta(\mathbf{x}^0), \\
    \hat{\mathbf{x}}^k &= \mathcal{M}_\theta(\hat{\mathbf{x}}^{k-1}) \quad \forall k \in \{2, \dots, K\}.
\end{align}

Conditional Denoising Diffusion Probabilistic Models (CDDPMs) learn the conditional distribution $p(\mathbf{x}^{k+1} \mid \mathbf{x}^k)$ by reversing a gradual noise-addition process. The sampling operator $\mathcal{M}_\theta$ iteratively refines a noise sample $\mathbf{y}_T \sim \mathcal{N}(\mathbf{0}, \mathbf{I})$ into a prediction $\hat{\mathbf{x}}^{k+1}$ conditioned on the previous state $\mathbf{x}^k$. A broader introduction can be found in Appendix \ref{prelim-appendix}. To simplify notation for the remainder of this section, we denote the condition as $\mathbf{x}$ and the target as $\mathbf{y}$. We adopt the $\boldsymbol{\epsilon}$-prediction formulation for the diffusion process, where at each step a denoiser $\boldsymbol{\epsilon}_\theta$ predicts the Gaussian noise added to the target $\mathbf{y} = \mathbf{y}_0$. The reverse process starts by sampling $\hat{\mathbf{y}}_{T} \sim \mathcal{N}(\mathbf{0}, \mathbf{I})$, and the following denoising steps are given by:
\begin{align}
    \hat{\mathbf{y}}_{t-1} &= \tilde{\mu}_t\!\left(\hat{\mathbf{y}}_t,\, \boldsymbol{\epsilon}_\theta(\hat{\mathbf{y}}_t, \mathbf{x}, \sigma_t)\right) + \sqrt{\tilde{\beta}_t}\,\boldsymbol{\epsilon}_{t-1},
    \label{eq:sampling}
\end{align}
where $t \in \{1, \dots, T\}$, $\sigma_t \triangleq \sqrt{1 - \bar{\alpha}_t}$ is the noise level at time $t$, $\boldsymbol{\epsilon}_t \sim \mathcal{N}(\mathbf{0}, \mathbf{I})$, and the posterior mean and variance are:
\begin{align}
    \tilde{\mu}_t(\mathbf{y}_t, \boldsymbol{\epsilon}_\theta) &= \frac{1}{\sqrt{\alpha_t}}\!\left(\mathbf{y}_t - \frac{\beta_t}{\sigma_t}\,\boldsymbol{\epsilon}_\theta\right), \label{eq:ddpm_mean}\\
    \tilde{\beta}_t &= \frac{1 - \bar{\alpha}_{t-1}}{1 - \bar{\alpha}_t}\,\beta_t. \label{eq:ddpm_var}
\end{align}
The model is trained to minimize the reconstruction loss:
\begin{align}
    \mathcal{L}_{\text{diff}}(\theta) &= \mathbb{E}_{\mathbf{x},\mathbf{y}, t, \boldsymbol{\epsilon}} \left[\| \boldsymbol{\epsilon}_\theta(\tilde{\mathbf{y}}_t, \mathbf{x}, \sigma_t) - \boldsymbol{\epsilon} \|^2 \right],
    \label{eq:l_diff}
\end{align}
where $\boldsymbol{\epsilon} \sim \mathcal{N}(\mathbf{0}, \mathbf{I})$ and $\tilde{\mathbf{y}}_t = \sqrt{\bar{\alpha}_t}\mathbf{y} + \sigma_t \boldsymbol{\epsilon}$ is the noisy ground-truth target. Given any noisy input $\mathbf{y}_t$ at noise level $\sigma_t$, the neural estimate of the clean signal is:
\begin{align}
    \mathbf{y}_{est}(\mathbf{y}_t, \mathbf{x}, \sigma_t) = \frac{\mathbf{y}_t - \sigma_t\boldsymbol{\epsilon}_\theta(\mathbf{y}_t, \mathbf{x}, \sigma_t)}{\sqrt{\bar{\alpha}_t}}.
\end{align}
During training, this is evaluated on the ground-truth noisy target $\tilde{\mathbf{y}}_t$; during inference, on the sampling iterate $\hat{\mathbf{y}}_t$.
Throughout this work, we use the following simplified one-step transition as a practical approximation to the full DDPM posterior~\eqref{eq:sampling}:
\begin{align}
    \hat{\mathbf{y}}_{t-1} = \sqrt{\bar{\alpha}_{t-1}}\,\mathbf{y}_{est}(\hat{\mathbf{y}}_{t}, \mathbf{x}, \sigma_t) + \sigma_{t-1}\,\boldsymbol{\epsilon}_t
    \label{eq:simplified_sampling}
\end{align}
i.e., the forward noising process applied to the predicted clean signal, bypassing the explicit dependence on $\hat{\mathbf{y}}_t$ in the mean. In the rest of this work, we will discard the dependency of $\mathbf{x}$ and $\sigma_t$ in $\mathbf{y}_\mathrm{est}(\cdot)$ to improve readability.

\subsection{Problem Definition: Minimizing Reconstruction Error}

While metrics like FID assess distributional quality, precision tasks such as PDE modeling require minimizing the specific trajectory error relative to the ground truth. In this setup, we aim to identify the schedule that leads to the optimal reconstruction error. We define the \textit{Optimal-Error Schedule} $\mathcal{S}^*$ as the configuration that minimizes the final inference reconstruction error:
\begin{align}
    \mathcal{S}^* = \operatorname*{argmin}_{\{\mathcal{S}\}} \mathbb{E}_{\mathbf{x}, \mathbf{y}, \boldsymbol{\epsilon}} \left[ \| \mathcal{M}_{\theta^*(\mathcal{S})}(\mathbf{x}) - \mathbf{y} \| \right],
    \label{eq:min_error}
\end{align}
where $\theta^{*}(\mathcal{S})$ denotes the model parameters converged via \eqref{eq:l_diff} under schedule $\mathcal{S}$. In the following, we fix a trained model and a pair $(\mathbf{x}, \mathbf{y})$, and omit these dependencies from the notation; all quantities are implicitly functions of $(\mathbf{x}, \mathbf{y})$, and $\theta$ is reintroduced from Proposition~\ref{prop:instability_threshold} onwards where model comparisons are needed. To analyze how the schedule influences (\ref{eq:min_error}), we first define the concept of reconstruction exposure-bias:

\begin{definition}[\textbf{Clean-Input vs. Inference-Input Error}]
    The \textbf{Clean-Input Error} measures prediction error when the state input is the \textit{noised ground-truth state} $\tilde{\mathbf{y}}_t$:
    \[
    \mathcal{E}^{\text{clean}}(t) = \mathbb{E}_{\boldsymbol{\epsilon}}  \| \mathbf{y}_{est}(\tilde{\mathbf{y}}_t) - \mathbf{y} \|
    \]

    The \textbf{Inference-Input Error} measures error when the state input is  the \textit{ancestor-sampling state} $\hat{\mathbf{y}}_t$:
    \[
    \mathcal{E}^{\text{inf}}(t) = \mathbb{E}_{\boldsymbol{\epsilon}} \left[ \| \mathbf{y}_{est}(\hat{\mathbf{y}}_t) - \mathbf{y} \| \right].
    \]
\end{definition}
The divergence between these two via the \textbf{Reconstruction Exposure-Bias (REB):}
\[\mathrm{REB}(t) \triangleq \frac{\mathcal{E}^{\text{inf}}(t)}{\mathcal{E}^{\text{clean}}(t)}\]

\subsection{Decomposing the REB: Two-Steps Bias as the Primary Driver}

We now analyze the sources of REB by introducing the \textit{Two-Steps Bias}, isolating the degradation that occurs in a single transition.

\begin{definition}[\textbf{Two-Steps Bias}]
    Define the \textit{two-steps noisy state} at time $t$ as renoising the estimate of $y$ obtained from a \textit{noised ground-truth state} as input at time $t+1$:
    \begin{equation}
        \hat{\mathbf{y}}_{t}^{(2S)} \triangleq \sqrt{\bar{\alpha}_{t}}\,\mathbf{y}_{est}(\tilde{\mathbf{y}}_{t+1}) + \sigma_{t}\,\boldsymbol{\epsilon}, \label{eq:y_2S} 
        \end{equation}
    The \textit{Two-Steps Bias} is defined as the ratio between the error of the second estimate and the clean-input error at time $t$:
    \begin{align}
        \mathcal{B}^{(2S)}(t) \triangleq \frac{\mathbb{E}_{\boldsymbol{\epsilon}}\|\mathbf{y}_{est}(\hat{\mathbf{y}}^{(2S)}_t) - \mathbf{y}\|}{\mathcal{E}^{\mathrm{clean}}(t)}
    \end{align}
\end{definition}

The following proposition shows that the REB is dominated by the individual two-steps biases.

\begin{proposition}[\textbf{Re-noising Attenuation}]
\label{prop:renoising_attenuation}
When estimated errors are nearly aligned, the following recursive bound holds:
\begin{align*}
    \mathrm{REB}(t) \;\lesssim\; \mathcal{B}^{(2S)}(t) \;+\; \lambda_t\,\mathrm{REB}(t+1),
\end{align*}
where $\lambda_t := \|J_t\|\cdot\sqrt{\bar{\alpha}_t}\cdot\frac{\|\mathbf{y}_{est}(\tilde{\mathbf{y}}_{t+1})-\mathbf{y}\|}{\|\mathbf{y}_{est}(\tilde{\mathbf{y}}_t)-\mathbf{y}\|}$. The REB is thus driven by the local two-step bias, and attenuated (resp. amplified) across steps when $\lambda_t < 1$ (resp. $\lambda_t > 1$). A full bound without the alignment assumption is given in Appendix~\ref{app:proof_renoising}.
\end{proposition}

The proof is given in Appendix~\ref{app:proof_renoising}, along with a visualization of the impact of the two-steps bias on the total REB for our practical datasets. This motivates enforcing $\mathcal{B}^{(2S)}(t) \leq \tau$ at each timestep as a practical proxy for controlling the full REB. We further decompose the two-steps bias into two distinct contributions:
\begin{align}
    \mathcal{B}^{(2S)}(t) = \mathcal{B}^{(own)}(t) + \delta(t), \label{decomposition}
\end{align}

\textbf{1. The Own-Prediction Bias} $\mathcal{B}^{(own)}$ measures the error amplification when the model is fed its own prediction instead of the ground-truth noisy input:
    \begin{equation}
        \hat{\mathbf{y}}_{t}^{(own)} \triangleq \sqrt{\bar{\alpha}_{t}}\,\mathbf{y}_{est}(\tilde{\mathbf{y}}_{t}) + \sigma_{t}\,\boldsymbol{\epsilon}, \label{eq:y_own} 
        \end{equation}
    \vspace{-5mm}
   \begin{align}
        \mathcal{B}^{(\mathrm{own})}(t) \triangleq \frac{\mathbb{E}_{\boldsymbol{\epsilon}}\|\mathbf{y}_{est}(\hat{\mathbf{y}}_{t}^{(own)}) - \mathbf{y}\|}{\mathcal{E}^{\mathrm{clean}}(t)}
    \end{align}
A model is \emph{stable} at a step $t$ when $\mathcal{B}^{(own)}(t) \approx 1$, meaning that repeated denoise-renoise iterations introduce no drift from the noised ground-truth state prediction. Conversely, $\mathcal{B}^{(own)}(t) \gg 1$ signals instability.

\textbf{2. The Propagation Residual} $\delta(t) > 0$ denotes the residual error due to additional error inherited from the step $t+1$.

Minimizing $\mathcal{B}^{(own)}$ is a necessary condition for stability. The following result establishes that own-prediction bias directly depends on clean-input error, and that for each noise level there exists a natural stability boundary.

\begin{proposition}[\textbf{Stability Threshold}]
\label{prop:instability_threshold}
    Assuming \textbf{(A1)} Wiener denoiser structure and $\textbf{(A2)}$ Spectral bias of the neural denoiser (detailed in Appendix~\ref{app:proof_instability}), $\mathcal{B}^{(own)}_\theta(t)$ is an increasing function of $\mathcal{E}^{\text{clean}}_{\theta}(t)$. In particular, at each noise level $\sigma_t$ and for any threshold $\tau \geq 1$, there exists a critical clean-input error $\gamma(\sigma_t, \tau)$, increasing in $\sigma_t$ and decreasing in $\tau$, such that if $\mathcal{E}^{\text{clean}}_{\theta}(\sigma_t) \leq \gamma(\sigma_t, \tau)$ then $\mathcal{B}^{(own)}_\theta(t) \leq \tau$.
\end{proposition}

\begin{figure}[t]
    \centering
    \includegraphics[width=\linewidth]{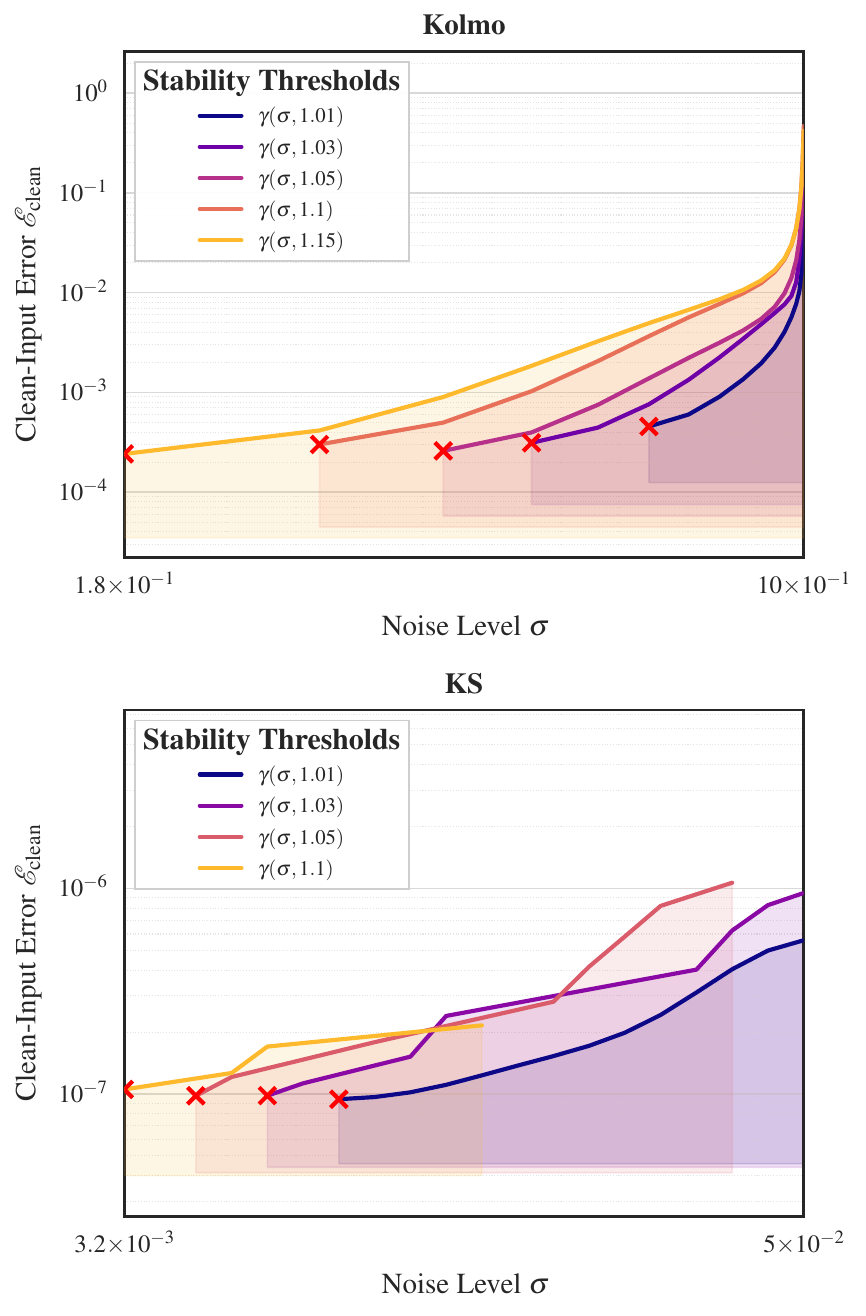}
    \caption{\textbf{Stability Thresholds.} We train a diffusion model on a log-uniform schedule. For a given training checkpoint, if the own-prediction bias at a given noise level falls in the $\tau\pm 0.05$ box, where $\tau \in [1.01,..,1.15]$, we report its clean-input error. One sees that Clean-Input Error and Own-Prediction Bias correlate on both datasets. The red crosses correspond to the smallest noise-level for which the model achieves $\mathbb{E}_{\mathbf{x}, \mathbf{y}}\mathcal{B}^{(own)} \leq \tau$ at convergence, indicating that there is a -- data and model-dependent -- minimal noise-level where the model can be bias-free.}
    \label{fig:instability_thresholds}
    \vspace{-5mm}
\end{figure}

The proof is given in Appendix~\ref{app:proof_instability}. In other words, at each noise level, there is a minimal clean-input error to be achieved in order for the model to be stable. $(A1)$ does not hold exactly in practice, and $(A2)$ may not be true for all classes of neural networks, however we display the empirical correlation between the clean-input error and the own-prediction bias for a diffusion model in Figure~\ref{fig:instability_thresholds}, on different benchmark datasets (see Section \ref{experimental-setup} for setup).

\subsection{Schedules control Clean-Input Error decrease rate}

In the rest of this work, we will restrict the schedule optimization problem to schedules that are stable. Given a small $\tau$, we therefore restrict (\ref{eq:min_error}) to identifying the \textit{Optimal Stable Schedule}:
\begin{align}\label{bias-free}
    \mathcal{S}^* &= \operatorname*{argmin}_{\{\mathcal{S}\}} \mathbb{E}_{\mathbf{x}, \mathbf{y}} \big [ \mathcal{E}_{\theta^*(\mathcal{S})}^{\text{clean}}\left(0\right) \big ] \\ \text{s.t.} \quad &\forall t < T, \quad \mathbb{E}_{\mathbf{x}, \mathbf{y}} \big[ \mathcal{B}_{\theta^*(\mathcal{S})}^{(2S)}(t) \big ]\leq \tau, \nonumber
\end{align}
The reason is that a model with large REB drifts outside of distribution, leading to unpredictable behaviour and making sub-optimal allocation of model capacity (reducing clean-input error on noise-levels where it cannot be reduced during inference). We therefore expect that the true optimal schedule falls in this category.

As illustrated in Figure \ref{fig:schedule_comparison}, the schedule governs the decay rate of the reconstruction error, directly influencing the final prediction. In the absence of any stability constraint, one could simply set $\sigma_0 \to 0$,  making the final denoising task trivially easy ($\mathcal{E}^{\text{clean}}(0) \to 0$). However, this degenerate schedule doesn't take into account that a larger input error is propagated during inference compared to training. Consequently, $\hat{\mathbf{y}}_0$ diverges from $\tilde{\mathbf{y}}_0$, leading to high REB. Using the stability threshold, we can formalize this:

\begin{proposition}[\textbf{Slow Error Decrease Principle}]
\label{prop:slow_decrease}
Under the stability condition of Proposition~\ref{prop:instability_threshold}, a finite capacity assumption \textbf{(A1)}, and assuming $\mathcal{E}^{\text{clean}}_\theta(\sigma)$ is strictly decreasing in $\sigma$ \textbf{(A2)}, the schedule minimizing (\ref{bias-free}) contains no unnecessary noise levels: every intermediate $\sigma_k$ is necessary in the sense that removing it would violate the bias constraint, i.e. lead to $\mathcal{B}^{(2S)}(k-1) > \tau$. In particular, the greedy construction that always takes the largest feasible jump is optimal.
\end{proposition}

This follows from the observation that any unnecessary intermediate step represents wasted capacity that could be reallocated to reduce $\mathcal{E}^{\text{clean}}_\theta(0)$ (a formal proof is given in Appendix~\ref{app:proof_slow_decrease}). Note that the bias constraint need not be tight at each step — it suffices that no larger jump is feasible. Since satisfying the bias constraint is increasingly easy at higher noise levels (as $\gamma(\sigma, \tau)$ increases with $\sigma$), the schedule is naturally coarse at high noise and fine near $\sigma_0$. This motivates reformulating (\ref{bias-free}) as minimizing the final noise level $\sigma_{0}$ subject to the bias constraint:
\vspace{-2mm}
\begin{align}
    \min \sigma_0 \quad \text{s.t.} \quad \forall t < T, \quad \;\mathbb{E}_{\mathbf{x}, \mathbf{y}} \; \mathcal{B}_\theta^{(2S)}(\mathbf{x}, \mathbf{y}, t) \leq \tau
\end{align}

\begin{figure}
    \includegraphics[width=\linewidth, trim=0 8pt 0 0pt, clip, page=1]{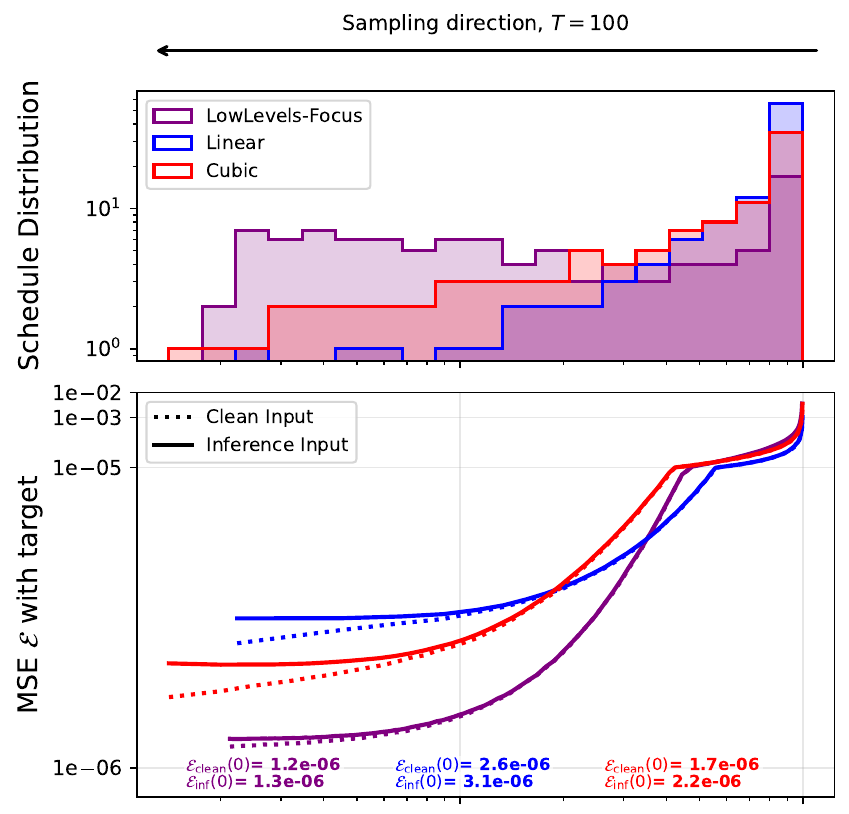}
    \vspace{-0.1mm}
    \includegraphics[width=\linewidth, trim=0 8pt 0 0pt, clip, page=2]{schedule_comparisons_2.pdf}
    \caption{\textbf{Effect of schedule on final reconstruction error}: Clean and Inference Input Errors during training. The diffusion sampling direction goes from right (high noise-levels) to left (low noise-levels). Schedules with more weight on low noise-levels obtain a better reconstruction performance early in the denoising process. However they are then unable to reduce the error further compared to schedules that put more focus on low noise-levels. At the end of sampling, schedules with higher clean-input reconstruction error start suffering from REB.}
    \label{fig:schedule_comparison}
    \vspace{-8mm}
\end{figure}

\section{Bias-Constrained Schedule Construction via Error Constraints}

Based on decomposition (\ref{decomposition}), a stable schedule must ensure that the sum of the own-prediction bias $\mathcal{B}^{(own)}$ and the propagation bias $\delta$ remain bounded, yielding the following principles:
\begin{itemize}
    \item At every $\sigma_t$ in $\mathcal{S}$, clean-input error must be reduced enough to reach stability, i.e.\ $\mathcal{B}^{(own)}(t) \leq \tau$.
    \item The step size $\sigma_{t+1} - \sigma_{t}$ must be constrained to limit the propagation bias such that ${B}^{(2S)}(t) \leq \tau$.
    \item \textbf{Efficient Capacity Allocation:} Driving the two-steps bias well below the stability boundary for a given noise-level is wasteful. Capacity should be conserved for lower noise levels.
\end{itemize}

\subsection{Adaptive Scheduling Algorithm}

Our algorithm contains two stages : an exploration stage, where the stability thresholds are discovered for the task at hand; and a schedule construction phase, where we greedily pick the noise levels such that the 2-steps bias remains bounded by $\tau$ at each step. While the ideal $\tau$ is 1, a value of $\tau$ slightly above 1 allows more flexibility in constructing the schedule and is shown experimentally to still lead to no error accumulation, potentially due to error correction at next step.

\paragraph{Phase 1: Exploration of Stability Thresholds.}
We initialize a dense log-uniform exploration schedule $\boldsymbol{\sigma}^{\text{exp}} = \{\sigma^{(1)}, \ldots, \sigma^{(N)}\}$ covering $[\sigma_{\min}, \sigma_T]$ with uniform weight. We train the model on this schedule, periodically evaluating $\mathcal{B}^{(\text{own})}(i)$ for all active levels. Any level satisfying $\mathcal{B}^{(\text{own})}(i) \leq \tau$ is \emph{solved}: we save the corresponding checkpoint $\theta^*(\sigma^{(i)})$ and remove $\sigma^{(i)}$ from the active schedule. Because solved levels are progressively removed, the active schedule shrinks continuously, making later epochs strictly cheaper than earlier ones. The exploration terminates once no new level is solved in a pass, ensuring the total epoch count does not exceed $E$ (the number of epochs required to train a single baseline model to convergence). This yields a set of checkpoints $\{\theta^*(\sigma)\}_{\sigma \in \boldsymbol{\sigma}^{\text{solved}}}$.

\paragraph{Phase 2: Bias-Constrained Schedule Construction.}
Given the solved checkpoints, we construct the final schedule greedily, by setting $\sigma_0 = \min \boldsymbol{\sigma}^{\text{solved}}$ and proceeding forward:
\begin{align}
    \sigma_{t+1} = \max \left\{ \sigma' \in \boldsymbol{\sigma}^{\text{solved}} : \mathcal{B}^{(2S)}_{\theta^*(\sigma'), \theta^*(\sigma_t)}(\sigma_t, \sigma') \leq \tau \right\},
\end{align}
where $\mathcal{B}^{(2S)}_{\theta^*(\sigma'), \theta^*(\sigma_t)}(\sigma_t, \sigma')$ denotes the two-steps bias evaluated using checkpoint $\theta^*(\sigma')$ for the first denoising step and $\theta^*(\sigma_t)$ for the second. At each step we take the largest feasible jump, directly implementing the slow decrease principle (Proposition~\ref{prop:slow_decrease}). This ensures (1) a schedule with the minimal number of diffusion steps and (2) no wasted model capacity on unnecessary intermediate steps. The number of steps $T$ emerges naturally from the construction. This phase requires only forward passes to evaluate two-steps biases between checkpoint pairs  (no gradient computation) and is therefore computationally negligible. The algorithm pseudo-code can be found in Appendix~\ref{implementation}. The model is trained from scratch on the constructed schedule using the standard diffusion loss~\eqref{eq:l_diff}.

\section{Proxy Unrolled Training}

Training neural emulators via \textit{Teacher Forcing} (TF) is notorious for leading to \textit{Simulation Exposure Bias} \cite{schmidt2019generalization, brandstetter2022message, chen2024diffusion}: during inference, the model must condition on its own past predictions $\hat{\mathbf{x}}^k$ rather than the ground-truth used in training. Small errors in $\hat{\mathbf{x}}^k$ shift the input distribution for subsequent steps, causing errors to accumulate and trajectories to diverge.

For fast-inference models, this limitation is typically mitigated via \textit{Unrolled Training} (UT) \cite{list2022learned}, where the model is optimized over a horizon of $U$ autoregressive steps. In the context of diffusion models, a naive unrolling takes the form:
\begin{align}
    \mathcal{L}_{\text{UT}}(\theta) = \mathbb{E}_t \sum_{u=1}^{U} \mathcal{L}_{\text{diff}} \left( \mathbf{x}^{k+u}, \mathcal{M}_\theta^{\circ u}\left(\mathbf{x}^k\right) \right).
    \label{eq:unrolled_loss}
\end{align}
However minimizing this loss is computationally prohibitive, as generating a single step $\hat{\mathbf{x}}^{k+1} = \mathcal{M}_\theta(\hat{\mathbf{x}}^k)$ necessitates executing the full iterative sampling chain (e.g., $T$ function evaluations). Therefore obtaining a simple push-forward estimate \cite{brandstetter2022message} is intractable, let alone backpropagating through the sampling chain.

\subsection{Estimating the Model's Output}

In this section, we assume that we have access to a diffusion model with low exposure-bias. We leverage this property for enabling computationally efficient unrolled training. In particular we introduce a \textit{Proxy Estimate} $\mathcal{P}^{(n)}_\theta$, which approximates the full model output using only the final $n$ denoising steps : 
\begin{align}
    \mathcal{P}^{(n)}_\theta(\mathbf{x}^k, \mathbf{x}^{k+1}) \triangleq \text{Denoise}_{n...0}\left(\tilde{\mathbf{x}}^{k+1}_n, \text{cond}=\mathbf{x}^k \right).
\end{align}
Where $\tilde{\mathbf{x}}^{k+1}_n = \sqrt{\bar{\alpha}_n} \ \mathbf{x}^{k+1} + \sqrt{1 - \bar{\alpha}_n} \ \boldsymbol{\epsilon}$. 
and $\text{Denoise}_{n...0}$ represents the sequence of $n$ denoising steps from $t=n$ down to $t=0$.

\paragraph{Faithfulness of the proxy estimate.} If the model does not suffer from exposure-bias up to step $n$, the distribution of latents obtained during sampling $\hat{\mathbf{x}}^{k+1}_n \sim p_\theta\left(\cdot \mid \mathbf{x}^{k}_n, \mathbf{z} \right)$ matches the ground-truth forward distribution $\tilde{\mathbf{x}}^{k+1}_n \sim q\left(\cdot \mid \mathbf{x}^{k+1}\right)$, and therefore the distribution of $\mathcal{P}^{(n)}_\theta$ matches that of $\mathcal{M}_\theta$. We give an experimental visualization of the accuracy of the gradients in Appendix~\ref{gradient-appendix}. We consequently define the 2-steps proxy unrolled-training loss as:
\begin{align*}
    \mathcal{L}_{\text{P-UT}}(\theta) &= \mathbb{E}_k \left[ \mathcal{L}_{\text{diff}}\!\left(\mathbf{x}^{k}, \mathbf{x}^{k+1}\right) \right] +  \\
    & \mathbb{E}_k \left[\mathcal{L}_{\text{diff}}\!\left(P_\theta(\mathbf{x}^{k}, \mathbf{x}^{k+1}), \mathbf{x}^{k+2}\right) \right]
\end{align*}

Although gradient propagation is truncated beyond the $n$ proxy steps, the use of a high-fidelity estimate provides a precise supervision signal, enabling the model to build resilience to its own prediction and correct potential artifact generation.

\section{Experimental Results}\label{experiments}

\subsection{Experimental Setup \label{experimental-setup}}

\textbf{Datasets.} We evaluate our method on three distinct fluid dynamics datasets representing different physical regimes. Visualizations are provided in Figure \ref{fig:benchmark}

\textbf{1D Kuramoto-Sivashinsky (KS)} \cite{brandstetter2022lie}: Fourth-order nonlinear PDE, which governs flame front propagation and chaotic solidification dynamics.  The scalar field $u$ evolves according to $\partial_\tau u + u\partial_x u + \partial_x^2 u + \nu \partial_x^4 u = 0$. Numerical integration is performed on a periodic domain with 256 spatial points and a timestep $\Delta\tau = 0.2$. We use 512 training trajectories of length $140\Delta\tau$ and 64 validation/testing trajectories of length $640\Delta\tau$. The models are trained with a step-size of $4\Delta \tau$.

\textbf{2D Transonic Flow (Tra)} \cite{kohl2023benchmarking} : Simulated flow over a cylinder on a $128 \times 64$ grid. Time evolving fields are 2D velocity $u$, pressure $p$, and density $\rho$. The evaluations assess model performance over $R=60$ timesteps. 

\textbf{2D Kolmogorov Flow (Kolmo)} \cite{rozet2023score} : Incompressible fluid driven by sinusoidal forcing.  The system obeys the Navier-Stokes equations subject to the incompressibility constraint $\nabla \cdot \mathbf{u} = 0.$ 800 training trajectories of length 64, and 100 trajectories for validation/testing are simulated, with a spatial resolution of $64 \times 64$ and $\Delta\tau = 0.2$.

\begin{figure*}
    \centering
    \includegraphics[width=\linewidth]{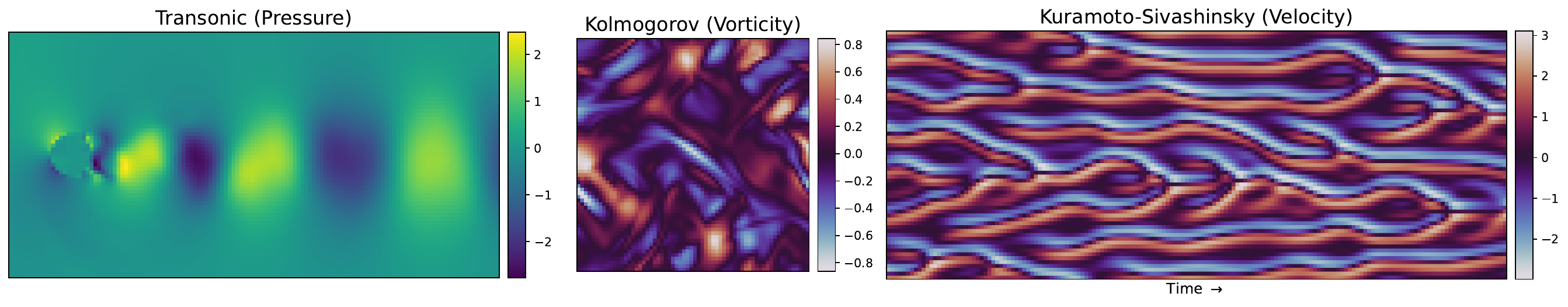}
    \caption{Samples from benchmark datasets}
    \label{fig:benchmark}
\end{figure*}

\textbf{Baselines.}
We train standard diffusion schedule baselines (Linear, Cosine, Sigmoid) using Teacher-Forcing loss. We further train deterministic U-Net baselines with Teacher-Forcing and Unrolled-Training with $U=2$ and $U=8$ (for each dataset we report the best performing U-Net training variant). Additionaly, we train PDE-Refiner \cite{lippe2023pde}, a variant of diffusion that focuses on very low noise-levels.

\textbf{Implementation Details.}
We use the same standard U-Net architecture with attention mechanisms following \cite{kohl2023benchmarking} for all models (all diffusion models and vanilla U-Net baselines). We employ continuous time-embeddings parameterized by the Signal-to-Noise Ratio (SNR) rather than discrete timesteps. 

\textbf{Training strategy.} The procedure contains two stages:\newline \textit{(1) Pre-training:} We first run the Adaptive Schedule Procedure to obtain a bias-constrained schedule for the given task. \newline\textit{(2) Proxy Unrolled Fine-tuning:} Since diffusion exposure bias is directly correlated with model error through the instability threshold, we only use the Unrolled proxy as a fine-tuning stage. This prevents introducing noisy gradients derived from inaccurate proxy estimates during the early phases of training, and saves computation time. We set the number of proxy steps to $n=1$.
\vspace{-3mm}
\paragraph{Computational Cost \& Hyperparameters.}
Each individual training run uses the same number of epochs $E$. For our adaptive schedule, the total number of epochs is therefore $2E$ as exploration and final training both contain a single training runs. All baseline methods are also trained to convergence over $E$ epochs. For each dataset, $E$ was picked as the largest budget which leads to improvement of the 1-step MSE. Full hyper-parameter configurations and architectural details are provided in Appendix \ref{implementation}.

\textbf{Metrics.}
We report mean squared errors for the first prediction step and the short-term (10 steps) prediction. For KS and Kolmo, long-term accuracy is evaluated on the ground-truth correlation of the predicted signal and the vorticity $\omega = \partial_x u_y - \partial_y u_x$, respectively, while for Tra we compute the MSE at the last time-step. To assess physical consistency in long-horizon rollouts for the 2D case, we report the Fréchet Spectral Distance (FSD) \cite{liu2025rolling}.

\subsection{Results and Analysis}

Comparative results across the different benchmarks are displayed in Table 1.

\begin{table*}[t]
\centering
\begin{small} 
\begin{tabular}{llcccc}
    \toprule
    Dataset & Method & 1-step MSE & 10-steps MSE & \shortstack{High-Correlation Time (s) / \\ Final MSE} & \shortstack{Last-step FSD} \\
    \midrule
    \textbf{Kolmo} 
        & U-Net UT, $U=2$ & 9.53e-7 & 2.79e-5 & 10.0 (6.5, 12.4) & 9.1e-1 \\
        & Linear TF & 1.30e-6 & 3.16e-4 & 9.7 (6.0, 12.3) & 2.1e5 \\
        & Sigmoid TF & 1.25e-6 & 9.81e7 & 9.5 (4.6, 12.1) & 2.4e6 \\
        & PDERefiner \cite{lippe2023pde} & 1.12e-6 & 8.16e-5 & 10.2 (7.9, 12.2) & 1.05e1 \\
        \cmidrule(l){2-6}
        & Adaptive (Ours) & \textbf{8.12e-7} & 2.18e-4 & 10.7 (7.6, 12.6) & 3.7e5 \\
        & Adaptive + Proxy UT, $U = 1$ (Ours) & \textbf{8.07e-7} & \textbf{1.72e-5} & \textbf{11.0} (9.0, 12.6) & \textbf{3.8e-1} \\
    \midrule
    \textbf{Tra} 
        & U-Net UT, $U=1$ & 5.63e-5 & 2.0e-5 & 7.50e-1 & 2.2e5 \\
        & U-Net UT, $U=8$ & 4.15e-4 & 2.54e-3 & \textbf{1.08e-1} & 6.3e1 \\
        & Linear TF & 5.76e-5 & 9.73e-4 & 6.1e-1 & 3.54e1 \\
        & Cosine TF & 4.90e-5 & 8.79e-4 & 1.22e-1 & 2.3e1 \\
        \cmidrule(l){2-6}
        & Adaptive (Ours) & \textbf{3.87e-5} & 8.11e-4 & 1.33e-1 & 3.3e1 \\
        & Adaptive + Proxy UT, $U=1$ (Ours) & 4.05e-5 & \textbf{6.66e-4} & 1.25e-1 & 3.8e1 \\
    \midrule
    \textbf{KS} 
        & U-Net TF & 1.60e-7 & 1.50e-5 & 68.4 (50.5, 80.0) & -- \\
        & Cosine TF & 2.79e-7 & 1.43e-5 & 74.9 (48.8, 105.4) & -- \\
        & Sigmoid TF & 1.55e-7 & 1.57e-5 & 78.4 (50.4, 113.4) & -- \\
        & PDERefiner \cite{lippe2023pde} & \textbf{8.99e-8} & 1.19e-5 & \textbf{88.0} (52.5, 123.6) & \\
        \cmidrule(l){2-6}
        & Adaptive (Ours) & 9.55e-8 & \textbf{9.74e-6} & 85.8 (55.7, 117.4) & -- \\
        & Adaptive + Proxy UT, $U = 1$ (Ours) & \textbf{8.83e-8} & 1.04e-5 & 86.1 (55.0, 111.7) & -- \\
    \bottomrule
\end{tabular}
\end{small}
\label{table:results}
\vspace{2mm}
\caption{Comparison of Method Performance on Fluid Dynamics Datasets. We compare standard deterministic baselines and diffusion schedules against our Adaptive + Unrolling framework. For High-Correlation Time, the numbers in parenthesis are the high-correlation time of the worst and best 10 trajectories. The bolded numbers are the smallest value across methods in a (dataset, metric) pair. For FSD, we only consider improvements to be significant if they lead to an order of magnitude improvement. On each dataset, we report the best-performing models among baselines.
\vspace{-8mm}
}
\end{table*}
\vspace{-1mm}
\paragraph{Impact of Adaptive Schedule.}
Our experiments show that the proposed schedule optimization framework is the dominant factor driving performance gains. Regarding one-step MSE, our method consistently outperforms standard diffusion baselines by a significant margin. This improvement improves deccorelation time For Kolmogorov Flow and KS which are especially sensitive to initial error. We note that PDE-Refiner~\cite{lippe2023pde} performs competitively on KS, achieving the best 1-step MSE and high-correlation time among all methods; this is consistent with its design, which explicitly focuses capacity on low noise levels — a strategy that aligns with our Slow Error Decrease Principle. Our method achieves comparable performance on KS while also generalising to the other benchmarks where PDE-Refiner's fixed low-noise focus is less effective. \newline\textbf{Visualisations of the obtained schedules.} Figure \ref{fig:GradientSimilarity} Left plot shows the reconstruction error landscape of the obtained schedule after running adaptive training over the 3 benchmark datasets. Inference-Input error nicely follows the Clean-input error, indicating that our proxy exposure-bias metrics yield consistent REB minimization. Furthermore, the obtained $\sigma_0$ differ vastly for each dataset. Following observations made in \cite{lippe2023pde}, a low optimal $\sigma_0$ appears to be correlated with more high-frequency components in the energy spectrum in the fluid data. However, we argue that the Energy spectrum is not the only driver of $\sigma_0$. In particular, modifications in step-size naturally influence task difficulty, and could therefore lead to different errors and minimal sigma. Investigating the main drivers of noise-level landscapes is of relevant importance, furthermore it remains to be investigated whether the instability thresholds $\gamma(\sigma, \tau)$ are shared quantities across tasks.

\paragraph{Impact of Proxy Unrolled Training.}
Consistent with the hypothesis that exposure bias degrades autoregressive rollouts, Proxy Unrolled Training significantly enhances both short-term precision and long-term stability. Notably, on Kolmogorov Flow, while the one-step error is similar as Teacher Forcing, our proxy method leads to an improved long-term correlation with the ground-truth trajectory. FSD improvements demonstrate the prevention of artifact formation and maintains bounded physical consistency for long-rollouts, compared to baseline models. This validates that the proxy estimate is sufficient to sensitize the model to its own distribution shifts. We provide the temporal evolution of the Fréchet Spectral Distance (FSD) in Appendix \ref{FSD-appendix}.

\begin{figure}[H]
    \centering
    \includegraphics[width=1\linewidth]{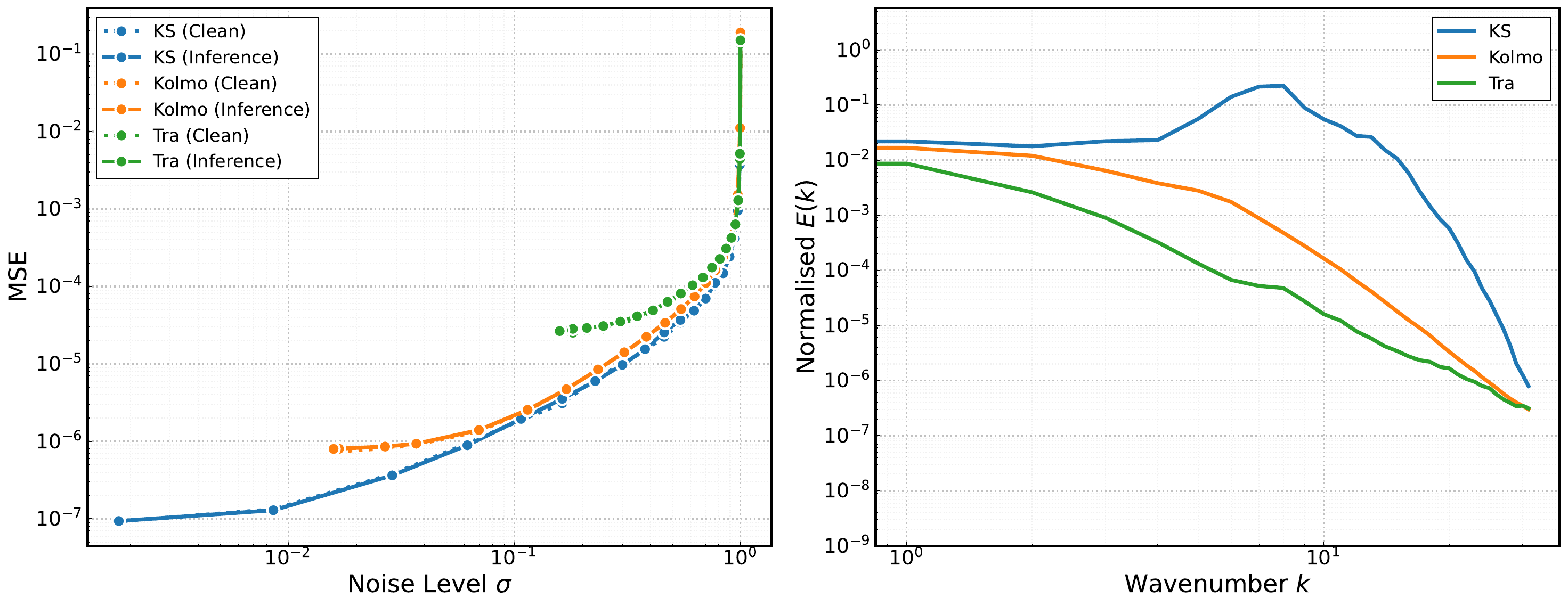}
    \caption{On the three benchmark datasets, the obtained adaptive schedules are, as shown in the left plot, exposure-bias free (inference error matches clean-input error). Interestingly, tasks which contain a higher high-frequency content tend to obtain a smaller minimal $\sigma_0$, matching the observations made in \cite{lippe2023pde}.}
    \label{fig:GradientSimilarity}
\end{figure}
\vspace{-10mm}
\section{Discussion and Future Work}

In this work, we have established an empirically grounded link between noise schedules, reconstruction error, and diffusion exposure bias for PDE diffusion models. We demonstrated that our adaptive schedule can yield near-optimal reconstruction, showing superior performance compared to both deterministic and probabilistic baselines. Furthermore, we proposed to leverage exposure-bias reduction and obtain a proxy unrolled training estimate, leading to considerable improvements in mitigating artifact formation. Beyond autoregressive simulations, our schedule holds promise for other reconstruction-bound scenarii, such as diverse inverse problems, data-assimilation or super-resolution tasks.

While our primary goal was to isolate the governing factors of sampling reconstruction error, one may be suspicious of the double training time required to idenfy a bias-constrained schedule. We claim, however, that our method can be considered a "smart hyperparameter exploration": given how large the space of schedules is, iterating over all possible schedules one by one would be much more computationally expensive. We showed that our method adapts well to different data statistics, with potential insights for larger-scale applications.

A fundamental question raised by our findings is the balance between modeling uncertainty and maintaining reconstruction accuracy. Traditional diffusion schedules in computer vision are optimized for perceptual diversity, often at the expense of strict pixel-wise accuracy \cite{blau2018perception}. Investigating the trade-off between those two objectives remains a compelling direction for future research.

Finally, whether the reconstruction error landscape is influenced by the choice of denoising prediction parameterization (e.g., $x_0$-prediction \cite{hoogeboom2023simple} and $v$-prediction \cite{salimans2022progressive}) is an important question. For instance, $x_0$-prediction typically yields smaller errors in early sampling steps, which may fundamentally alter the dynamics. Another potential direction would be to evaluate if the type of exposure bias we have shed light on, which is, different from the bias identified in \cite{ning2023elucidating}, rather bound to allocation of model capacity, is also present in unconditional diffusion models.

\newpage

\section{Impact Statement}

This paper displays methodologies for enhancing probabilistic fluid models. While this could be of societal benefit for forecasting tasks such as weather and climate, we acknowledge that high-fidelity fluid simulations are general-purpose tools that can also be applied to military contexts, such as aerodynamics for defense technologies.

\bibliography{references}

\onecolumn

\appendix

\section{Preliminaries}\label{prelim-appendix}

\subsection{Conditional Diffusion Models}
\label{sec:prelim_diffusion}
\textit{Denoising Diffusion Probabilistic Models} (DDPM) are generative models that learn to approximate a data distribution $p(\mathbf{y})$ by reversing a gradual noising process. We consider the conditional setting where the generation of a target $\mathbf{y} \in \mathbb{R}^d$ is conditioned on an input $\mathbf{x}$.

Given a fixed \textit{variance schedule} $\{\beta_t \in (0, 1)\}_{t=1}^T$, the \textit{forward process} $q(\tilde{\mathbf{y}}_{1:T} | \mathbf{y}_0)$ transforms a clean sample $\mathbf{y}_0 \sim p(\mathbf{y}|\mathbf{x})$ into Gaussian noise $\tilde{\mathbf{y}}_T \sim \mathcal{N}(\mathbf{0}, \mathbf{I})$ via a Markov chain:
\begin{align}
    q(\tilde{\mathbf{y}}_t | \tilde{\mathbf{y}}_{t-1}) &= \mathcal{N}(\tilde{\mathbf{y}}_t; \sqrt{1-\beta_t}\tilde{\mathbf{y}}_{t-1}, \beta_t \mathbf{I}).
\end{align}
A notable property of this process is that any intermediate noisy state $\tilde{\mathbf{y}}_t$ can be sampled directly from $\mathbf{y}_0$:
\begin{align}
    \tilde{\mathbf{y}}_t = \sqrt{\bar{\alpha}_t}\mathbf{y}_0 + \sqrt{1-\bar{\alpha}_t}\boldsymbol{\epsilon}, \quad \boldsymbol{\epsilon} \sim \mathcal{N}(\mathbf{0}, \mathbf{I}),
    \label{eq:forward_process}
\end{align}
where $\alpha_t \triangleq 1 - \beta_t$ and $\bar{\alpha}_t \triangleq \prod_{s=1}^t \alpha_s$. We characterize the noise intensity at step $t$ by the standard deviation $\sigma_t \triangleq \sqrt{1 - \bar{\alpha}_t}$.

The \textit{generative process} $p_\theta(\mathbf{y}_{0:T} | \mathbf{x})$ learns to reverse this corruption. A neural network $\mathbf{y}_{est}$ is trained to predict the clean signal $\mathbf{y}_0$ (or equivalently the noise $\boldsymbol{\epsilon}$) given a noisy latent $\tilde{\mathbf{y}}_t$, the conditioning $\mathbf{x}$, and the time step $t$. The model is optimized by minimizing the re-weighted squared error:
\begin{align}
    \mathcal{L}_{\text{diff}}(\theta) &= \mathbb{E}_{\mathbf{x}, \mathbf{y}_0, t, \boldsymbol{\epsilon}} \left[ \| \mathbf{y}_{est}(\tilde{\mathbf{y}}_t, \mathbf{x}, \sigma_t) - \mathbf{y}_0 \|^2 \right].
    \label{eq:diff_loss}
\end{align}

\subsection{Autoregressive Neural Emulators}
In the context of fluid dynamics, we aim to construct a surrogate model $\mathcal{M}_\theta$ that emulates the temporal evolution of a physical system. Given an initial state $\mathbf{x}^0$, the objective is to generate a trajectory $\{\hat{\mathbf{x}}^1, \dots, \hat{\mathbf{x}}^K\}$ that approximates the ground-truth sequence $\{\mathbf{x}^1, \dots, \mathbf{x}^K\}$. This is typically formulated as an autoregressive problem:
\begin{align}
    \hat{\mathbf{x}}^k = \mathcal{M}_\theta(\hat{\mathbf{x}}^{k-1}) \quad \text{for } k \in \{1, \dots, K\},
\end{align}
with $\hat{\mathbf{x}}^0 = \mathbf{x}^0$.

Standard training relies on \textit{Teacher Forcing} (TF), which minimizes the one-step prediction error conditioned on ground-truth states:
\begin{align}
    \mathcal{L}_{\text{TF}}(\theta) = \mathbb{E}_k \left[ \| \mathcal{M}_\theta(\mathbf{x}^k) - \mathbf{x}^{k+1} \|^2 \right].
\end{align}
However, TF is susceptible to \textit{exposure bias}: during inference, the model accumulates small errors, leading to a distributional shift in the input $\hat{\mathbf{x}}^k$ that diverges from the training distribution. To mitigate this, \textit{Unrolled Training} (UT) optimizes the model over a horizon of $M$ autoregressive steps:
\begin{align}
    \mathcal{L}_{\text{UT}}(\theta) &= \mathbb{E}_k \sum_{m=1}^{M} \| \hat{\mathbf{x}}^{k+m} - \mathbf{x}^{k+m} \|^2,
\end{align}
where $\hat{\mathbf{x}}^{k+m}$ is recursively generated from the model's own prior prediction $\hat{\mathbf{x}}^{k+m-1}$.

\newpage

\section{Proof of Proposition \ref{prop:renoising_attenuation}}
\label{app:proof_renoising}

\begin{proposition}[\textbf{Re-noising Attenuation}, restated]
When estimated errors are nearly aligned, the following recursive bound holds:
\begin{align*}
    \mathrm{REB}(t) \;\lesssim\; \mathcal{B}^{(2S)}(t) \;+\; \lambda_t\,\mathrm{REB}(t+1),
\end{align*}
where $\lambda_t := \|J_t\|\cdot\sqrt{\bar{\alpha}_t}\cdot\frac{\|\mathbf{y}_{est}(\tilde{\mathbf{y}}_{t+1})-\mathbf{y}\|}{\|\mathbf{y}_{est}(\tilde{\mathbf{y}}_t)-\mathbf{y}\|}$. The REB is thus driven by the local two-step bias, and attenuated (resp. amplified) across steps when $\lambda_t < 1$ (resp. $\lambda_t > 1$).
\end{proposition}

\medskip\noindent
\textbf{Full bound.} Let $\rho_1 = \cos\theta_1$ where $\theta_1$ is the angle between $\mathbf{y}_{est}(\hat{\mathbf{y}}_t^{(2S)})-\mathbf{y}$ and $\mathbf{y}_{est}(\tilde{\mathbf{y}}_t)-\mathbf{y}$, and let $\rho_2 = \cos\theta_2$ where $\theta_2$ is the angle between $\mathbf{y}_{est}(\hat{\mathbf{y}}_t)-\mathbf{y}$ and $\mathbf{y}_{est}(\tilde{\mathbf{y}}_t)-\mathbf{y}$. Then:
\begin{align}
    \mathrm{REB}(t)
    \;\leq\;
    1
    \;+\; \sqrt{\bigl(\mathcal{B}^{(2S)}(t) - \rho_1\bigr)^2 + (1 - \rho_1^2)}
    \;+\; \lambda_t \cdot \sqrt{\bigl(\mathrm{REB}(t+1) - \rho_2\bigr)^2 + (1 - \rho_2^2)}.
    \label{eq:REB_full}
\end{align}
When $\rho_1, \rho_2 \approx 1$, the square-root terms simplify to $\mathcal{B}^{(2S)}(t)$ and $\mathrm{REB}(t+1)$ respectively, recovering the statement above.

\medskip\noindent
\textbf{Sampling definitions.} Both inference and ground-truth paths follow one-step backward diffusion:
\begin{align}
    \hat{\mathbf{y}}_{t-1}^{(2S)} &= \sqrt{\bar{\alpha}_{t-1}}\,\mathbf{y}_{est}(\tilde{\mathbf{y}}_t) + \sigma_{t-1}\,\boldsymbol{\epsilon}, \label{eq:y_2S} \\
    \hat{\mathbf{y}}_{t-1} &= \sqrt{\bar{\alpha}_{t-1}}\,\mathbf{y}_{est}(\hat{\mathbf{y}}_t) + \sigma_{t-1}\,\boldsymbol{\epsilon}, \label{eq:y_inf} \\
    \hat{\mathbf{y}}_T &\sim \mathcal{N}(\mathbf{0}, \mathbf{I}), \label{eq:y_init} \\
    \tilde{\mathbf{y}}_{t} &= \sqrt{\bar{\alpha}_t}\, \mathbf{y} + \sigma_t\, \boldsymbol{\epsilon}, \label{eq:y_true}
\end{align}
where $\boldsymbol{\epsilon} \sim \mathcal{N}(\mathbf{0}, \mathbf{I})$ is the \textbf{same noise draw} shared between equations~\eqref{eq:y_2S} and~\eqref{eq:y_inf}. Similarly, $\hat{\mathbf{y}}_t$ and $\hat{\mathbf{y}}_t^{(2S)}$ use the same noise at their respective levels.

\medskip\noindent
\textbf{The two-steps and reconstruction bias are given as:}
\begin{align}
    \mathcal{B}^{(2S)}(t) \;\triangleq\; \frac{\bigl\| \mathbf{y}_{est}(\hat{\mathbf{y}}_t^{(2S)}) - \mathbf{y} \bigr\|}{\bigl\| \mathbf{y}_{est}(\tilde{\mathbf{y}}_t) - \mathbf{y}\bigr\|}, \label{eq:def_B2S}
\end{align}
\begin{align}
    \mathrm{REB}(t) \;\triangleq\; \frac{\bigl\| \mathbf{y}_{est}(\hat{\mathbf{y}}_t) - \mathbf{y} \bigr\|}{\bigl\| \mathbf{y}_{est}(\tilde{\mathbf{y}}_t) - \mathbf{y}\bigr\|}. \label{eq:def_REB}
\end{align}

\medskip\noindent
\textbf{Alternative biases definitions: }We furthemore define an alternative version of the biases taking into account the distance of estimated values with each other, rather than only their distance to $\mathbf{y}$.
\begin{align}
    \mathcal{B}^{(2S)}_\mathrm{alt}(t) \;\triangleq\; \frac{\bigl\| \mathbf{y}_{est}(\hat{\mathbf{y}}_t^{(2S)}) - \mathbf{y}_{est}(\tilde{\mathbf{y}}_t) \bigr\|}{\bigl\| \mathbf{y}_{est}(\tilde{\mathbf{y}}_t) - \mathbf{y}\bigr\|}, \label{eq:def_B2S}
\end{align}
\begin{align}
    \mathrm{REB}_\mathrm{alt}(t) \;\triangleq\; \frac{\bigl\| \mathbf{y}_{est}(\hat{\mathbf{y}}_t) - \mathbf{y}_{est}(\tilde{\mathbf{y}}_t) \bigr\|}{\bigl\| \mathbf{y}_{est}(\tilde{\mathbf{y}}_t) - \mathbf{y}\bigr\|}. \label{eq:def_REB}
\end{align}

\begin{proof}

We have that:

\begin{equation}
    REB_\mathrm{alt}(t) = \frac{\bigl\|\mathbf{y}_{est}(\hat{\mathbf{y}}_t) - \mathbf{y}_{est}(\tilde{\mathbf{y}}_t)\bigr\|}{\|\mathbf{y}_{est}(\tilde{\mathbf{y}}_t) - \mathbf{y}\|}
    \;\leq\;
    \frac{\bigl\|\mathbf{y}_{est}(\hat{\mathbf{y}}_t^{(2S)}) - \mathbf{y}_{est}(\tilde{\mathbf{y}}_t)\bigr\|}{\|\mathbf{y}_{est}(\tilde{\mathbf{y}}_t) - \mathbf{y}\|}
    \;+\;
    \frac{\bigl\|\mathbf{y}_{est}(\hat{\mathbf{y}}_t) - \mathbf{y}_{est}(\hat{\mathbf{y}}_t^{(2S)})\bigr\|}{\|\mathbf{y}_{est}(\tilde{\mathbf{y}}_t) - \mathbf{y}\|}
\label{eq:triangle_eq}\end{equation}

\medskip\noindent
Both $\hat{\mathbf{y}}_t$ and $\hat{\mathbf{y}}_t^{(2S)}$ are obtained by a backward diffusion step with the same noise draw, so:
\begin{equation}
    \hat{\mathbf{y}}_t - \hat{\mathbf{y}}_t^{(2S)} \;=\; \sqrt{\bar{\alpha}_t}\,\bigl[\mathbf{y}_{est}(\hat{\mathbf{y}}_{t+1}) - \mathbf{y}_{est}(\tilde{\mathbf{y}}_{t+1})\bigr].
\end{equation}
A first-order Jacobian bound then gives:
\begin{equation}
    \bigl\|\mathbf{y}_{est}(\hat{\mathbf{y}}_t) - \mathbf{y}_{est}(\hat{\mathbf{y}}_t^{(2S)})\bigr\|
    \;\lesssim\;
    \|J_t\|\cdot\sqrt{\bar{\alpha}_t}\cdot
    \bigl\|\mathbf{y}_{est}(\hat{\mathbf{y}}_{t+1}) - \mathbf{y}_{est}(\tilde{\mathbf{y}}_{t+1})\bigr\|,
\end{equation}
where $J_t := \nabla\mathbf{y}_{est}|_{\hat{\mathbf{y}}_t^{(2S)}}$.

\medskip\noindent
By definition, $\|\mathbf{y}_{est}(\hat{\mathbf{y}}_{t+1}) - \mathbf{y}_{est}(\tilde{\mathbf{y}}_{t+1})\| = \mathrm{REB}_\mathrm{alt}(t+1)\cdot\|\mathbf{y}_{est}(\tilde{\mathbf{y}}_{t+1})-\mathbf{y}\|$. Therefore (\ref{eq:triangle_eq}) gives :
\begin{equation}
    \mathrm{REB}_\mathrm{alt}(t) \;\lesssim\; \mathcal{B}^{(2S)}_\mathrm{alt}(\sigma_t,\sigma_{t+1}) \;+\; \lambda_t\cdot\mathrm{REB}_\mathrm{alt}(t+1).
    \label{eq:REB_recursive}
\end{equation}
with
\begin{equation}
    \lambda_t \;:=\; \|J_t\|\cdot\sqrt{\bar{\alpha}_t}\cdot\frac{\bigl\|\mathbf{y}_{est}(\tilde{\mathbf{y}}_{t+1})-\mathbf{y}\bigr\|}{\bigl\|\mathbf{y}_{est}(\tilde{\mathbf{y}}_t)-\mathbf{y}\bigr\|}
\end{equation}

\medskip\noindent
\textbf{Relationship between the two definitions.}
Let $a = \mathbf{y}_{est}(\hat{\mathbf{y}}_t^{(2S)}) - \mathbf{y}$ and $b = \mathbf{y}_{est}(\tilde{\mathbf{y}}_t) - \mathbf{y}$, so that the standard definition gives $\mathcal{B}^{(2S)} = \|a\|/\|b\|$ and the alternative gives $\mathcal{B}^{(2S)}_{\mathrm{alt}} = \|a - b\|/\|b\|$.
Letting $\theta$ denote the angle between $a$ and $b$, the law of cosines yields:
\begin{equation}
    \bigl(\mathcal{B}^{(2S)}_{\mathrm{alt}}\bigr)^2
    \;=\;
    \bigl(\mathcal{B}^{(2S)}\bigr)^2 + 1 - 2\,\mathcal{B}^{(2S)}\cos\theta \;=\; \bigl(\mathcal{B}^{(2S)} - \cos \theta \bigr)^2 + (1 - \cos \theta^2),
    \label{eq:alt_vs_std}
\end{equation}

Combining equation~\eqref{eq:alt_vs_std} and the triangle inequality:
\begin{equation}
    \mathcal{B}^{(2S)}(t) - 1
    \;\leq\;
    \mathcal{B}^{(2S)}_{\mathrm{alt}}(t)
    =
    \sqrt{\bigl(\mathcal{B}^{(2S)}(t) - \rho_1\bigr)^2 + (1 - \rho_1^2)},
\end{equation}

The same inequalities applie for the REB:
\begin{equation}
    \mathrm{REB}(t) - 1
    \;\leq\;
    \mathrm{REB}_{\mathrm{alt}}(t) \;\leq\; \sqrt{\bigl(\mathrm{REB}(t) - \rho_2\bigr)^2 + (1 - \rho_2^2)},
\end{equation}
where $\rho_2 = \cos \theta_2$ and $\theta_2$ is the angle between $\mathbf{y}_{est}(\hat{\mathbf{y}}_t) - \mathbf{y}$ and $\mathbf{y}_{est}(\tilde{\mathbf{y}}_t) - \mathbf{y}$.

Therefore we can express the standard REB in terms of the standard 2-steps bias and the next-step REB:
\begin{align}
    \mathrm{REB}(t) &\;\leq\; 1 + \mathrm{REB}_{\mathrm{alt}}(t)\\ 
    \;&\leq\; 1 + \mathcal{B}^{(2S)}_{\mathrm{alt}}(t) \;+\; \lambda_t\cdot\mathrm{REB}_{\mathrm{alt}}(t+1)\\
    \;&\leq\; 1 \;+\; \sqrt{\bigl(\mathcal{B}^{(2S)}(t) - \rho_1\bigr)^2 + (1 - \rho_1^2)} \;+\; \lambda_t \cdot \sqrt{\bigl({REB}(t+1) - \rho_2\bigr)^2 + (1 - \rho_2^2)}
    \label{eq:REB_std_recursive}
\end{align}
\end{proof}

\subsection{Visualization}

\begin{figure}[H]
    \centering
    \includegraphics[width=0.5\linewidth]{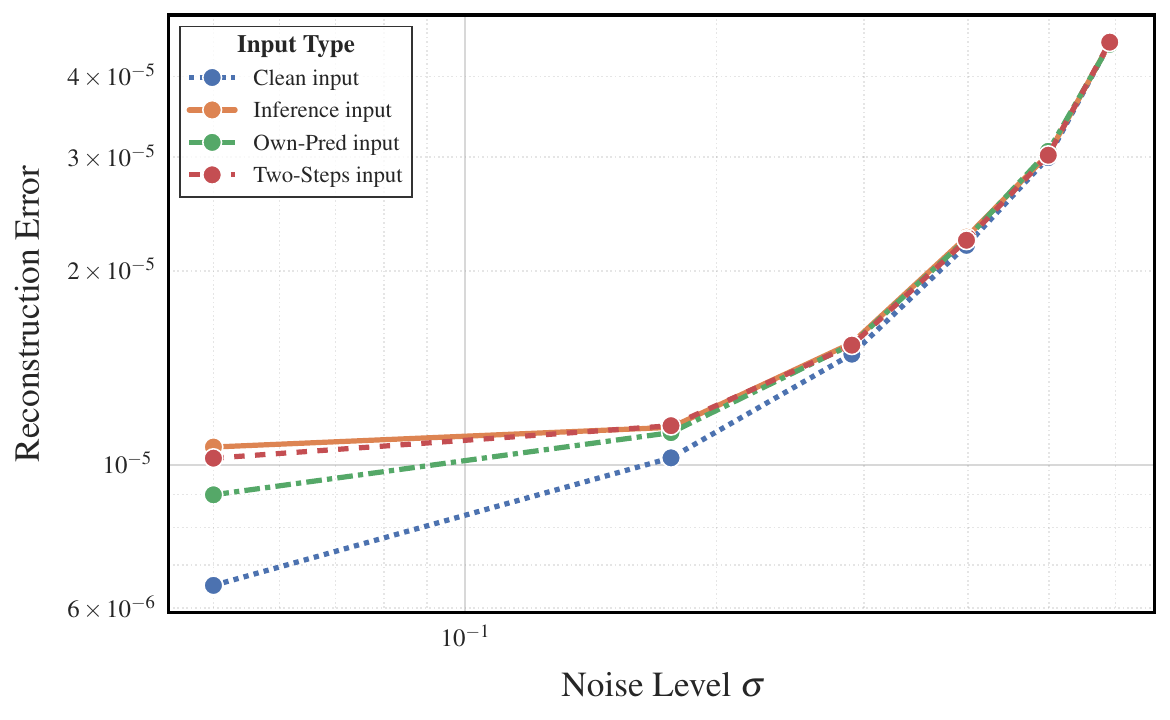}
    \caption{\textbf{Contributions of errors to the REB.} Metrics are reported for the final diffusion steps of a linear schedule model. Two-steps error nearly corresponds to inference-input error even though it only contains errors propagated from current and previous step. Own-prediction error is the major driver of the two-steps bias.}
    \label{fig:Own_pred}
\end{figure}

\newpage

\section{Proof of Proposition~\ref{prop:instability_threshold}}
\label{app:proof_instability}

\begin{proposition}[\textbf{Stability Threshold}, restated]
Assuming \textbf{(A1)} Wiener denoiser structure and \textbf{(A2)} Spectral bias of the neural denoiser (detailed below), $\mathcal{B}^{(own)}_\theta(t)$ is an increasing function of $\mathcal{E}^{\text{clean}}_{\theta}(t)$. In particular, at each noise level $\sigma_t$ and for any threshold $\tau \geq 1$, there exists a critical clean-input error $\gamma(\sigma_t, \tau)$, increasing in $\sigma_t$ and decreasing in $\tau$, such that if $\mathcal{E}^{\text{clean}}_{\theta}(\sigma_t) \leq \gamma(\sigma_t, \tau)$ then $\mathcal{B}^{(own)}_\theta(t) \leq \tau$.
\end{proposition}

\medskip\noindent
The proposition follows from the more precise result below, which makes assumptions (A1) and (A2) explicit.

\begin{proposition}[\textbf{Bias reduction through error-spectral shift}]
Let $\theta_A$ and $\theta_B$ be two denoisers and fix a noise level $\sigma$.  Let $J_\theta$ denote the Jacobian of $\mathbf{y}_{est}$ with respect to its noisy input, and let $E_\theta(k)$ denote the per-wavenumber residual error variance $E_\theta(k) = \mathbb{E}[|\hat\varepsilon_k|^2/N]$, where $\hat\varepsilon_k$ is the Fourier coefficient of the clean-input residual $\mathbf{y}_{est}(\tilde{\mathbf{y}}_\sigma, \mathbf{x}, \sigma) - \mathbf{y}$.
\begin{enumerate}
\item \textbf{(Own-Prediction Bias under Wiener denoiser structure.)} Assuming that $J_\theta$ is diagonal in the Fourier basis with eigenvalue $\lambda_k = (\sqrt{\bar\alpha}/\sigma^2)\,E_\theta(k)$ at each wavenumber $k$, then
\begin{equation}
    \mathcal{B}^{(own)}_\theta(\sigma) \;\approx\; 1 \;+\; \frac{2\bar\alpha}{\sigma^2}\;\frac{\sum_k E_\theta(k)^2}{\sum_k E_\theta(k)} \;+\; \left(\frac{\bar\alpha}{\sigma^2}\right)^{\!2}\;\frac{\sum_k E_\theta(k)^3}{\sum_k E_\theta(k)}.
    \label{eq:bias_error_psd_app}
\end{equation}
\item \textbf{(Spectral bias implies bias reduction.)} Suppose both models satisfy the Wiener structure of (1), that $E_B(k)$ is decreasing in $k$ (larger residual errors at low wavenumbers), and that $0 \leq E_A(k) \leq E_B(k)$ for all $k$ with the ratio $r_k := E_A(k)/E_B(k)$ non-decreasing in $k$ (the relative error reduction is larger at low wavenumbers). Then $R(\theta_A) \leq R(\theta_B)$, and hence $\mathcal{B}^{(own)}_{\theta_A}(\sigma) \leq \mathcal{B}^{(own)}_{\theta_B}(\sigma)$.
\end{enumerate}
\end{proposition}

\begin{proof}
\textbf{Part (1).}  Let $\boldsymbol{\varepsilon} = \mathbf{y}_{est}(\tilde{\mathbf{y}}_\sigma, \mathbf{x}, \sigma) - \mathbf{y}$ denote the clean-input residual. In the own-prediction setting~\eqref{own-prediction-bias}, the model receives the re-noised input $\hat{\mathbf{y}}_\sigma = \sqrt{\bar\alpha}\,\mathbf{y}_{est}(\tilde{\mathbf{y}}_\sigma, \mathbf{x}, \sigma) + \sigma\,\boldsymbol{z}$ instead of the ground-truth $\tilde{\mathbf{y}}_\sigma = \sqrt{\bar\alpha}\,\mathbf{y} + \sigma\,\boldsymbol{z}$. The input perturbation is thus $\boldsymbol{\delta} = \hat{\mathbf{y}}_\sigma - \tilde{\mathbf{y}}_\sigma = \sqrt{\bar\alpha}\,\boldsymbol{\varepsilon}$.

\medskip\noindent By first-order Taylor Expansion around $\tilde{\mathbf{y}}_\sigma$, we can write:
\begin{align}
    \mathbf{y}_{est}(\hat{\mathbf{y}}_\sigma, \mathbf{x}, \sigma) - \mathbf{y}
    &\approx \boldsymbol{\varepsilon} + J_\theta\,\boldsymbol{\delta}
    = \boldsymbol{\varepsilon} + \sqrt{\bar\alpha}\,J_\theta\,\boldsymbol{\varepsilon}.
\end{align}

Taking the squared norm and dividing by $\|\boldsymbol{\varepsilon}\|^2$, we obtain:
\begin{equation}
    \mathcal{B}^{(own)}_\theta(\sigma)
    = \frac{\|\boldsymbol{\varepsilon} + \sqrt{\bar\alpha}\,J_\theta\,\boldsymbol{\varepsilon}\|^2}{\|\boldsymbol{\varepsilon}\|^2}
    \;=\; 1 \;+\; \frac{2\sqrt{\bar\alpha}}{\|\boldsymbol{\varepsilon}\|^2}\,\boldsymbol{\varepsilon}^\top J_\theta\,\boldsymbol{\varepsilon} \;+\; \frac{\bar\alpha}{\|\boldsymbol{\varepsilon}\|^2}\,\|J_\theta\,\boldsymbol{\varepsilon}\|^2
    \label{eq:bias_full_taylor}
\end{equation}

Under the Wiener structure, $J_\theta$ is diagonal in the Fourier basis with eigenvalue $\lambda_k = (\sqrt{\bar\alpha}/\sigma^2)\,E_\theta(k)$ at wavenumber $k$.  By Parseval's theorem:
\begin{align}
    \boldsymbol{\varepsilon}^\top J_\theta\,\boldsymbol{\varepsilon}
    &= \frac{1}{N}\sum_k \lambda_k\,|\hat\varepsilon_k|^2, &
    \|J_\theta\,\boldsymbol{\varepsilon}\|^2
    &= \frac{1}{N}\sum_k \lambda_k^2\,|\hat\varepsilon_k|^2, &
    \|\boldsymbol{\varepsilon}\|^2
    &= \frac{1}{N}\sum_k |\hat\varepsilon_k|^2.
\end{align}

Since $\mathbb{E}[|\hat\varepsilon_k|^2/N] = E_\theta(k)$, we replace $|\hat\varepsilon_k|^2$ by its expectation $N\,E_\theta(k)$ to obtain:
\begin{align}
    \frac{2\sqrt{\bar\alpha}\,\boldsymbol{\varepsilon}^\top J_\theta\,\boldsymbol{\varepsilon}}{\|\boldsymbol{\varepsilon}\|^2}
    &\;\approx\; \frac{2\sqrt{\bar\alpha}\sum_k \lambda_k\,E_\theta(k)}{\sum_k E_\theta(k)}
    = \frac{2\bar\alpha}{\sigma^2}\;\frac{\sum_k E_\theta(k)^2}{\sum_k E_\theta(k)}, \\[4pt]
    \frac{\bar\alpha\,\|J_\theta\,\boldsymbol{\varepsilon}\|^2}{\|\boldsymbol{\varepsilon}\|^2}
    &\;\approx\; \frac{\bar\alpha\sum_k \lambda_k^2\,E_\theta(k)}{\sum_k E_\theta(k)}
    = \left(\frac{\bar\alpha}{\sigma^2}\right)^{\!2}\;\frac{\sum_k E_\theta(k)^3}{\sum_k E_\theta(k)},
\end{align}
where we substituted $\lambda_k = (\sqrt{\bar\alpha}/\sigma^2)\,E_\theta(k)$.  Combining both terms with~\eqref{eq:bias_full_taylor} yields~\eqref{eq:bias_error_psd_app}.

\textbf{Part (2).}  By~\eqref{eq:bias_error_psd_app}, $\mathcal{B}^{(own)}_\theta$ is determined (up to monotone transformations) by the ratio
\[
    R(\theta) \;:=\; \frac{\sum_k E_\theta(k)^2}{\sum_k E_\theta(k)},
\]
which is the $E_\theta(k)$-weighted mean of $E_\theta(k)$.  (The cubic term $\sum_k E_\theta(k)^3/\sum_k E_\theta(k)$ behaves analogously; we focus on $R$ for clarity.)  It therefore suffices to show that $R(\theta_A) \leq R(\theta_B)$.

\medskip\noindent
Write $E_A(k) = r_k\,E_B(k)$ with $r_k \in [0,1]$ non-decreasing.  Since $R(\theta_A) = \sum_k r_k^2\,E_B(k)^2 \big/ \sum_k r_k\,E_B(k)$, we have :
\begin{equation*}
\begin{aligned}
    R(\theta_A) &= \frac{\sum_k r_k^2\,E_B(k)^2}{\sum_k r_k E_B(k)} \\
    &\leq \frac{\sum_k r_k\,E_B(k)^2}{\sum_k r_k E_B(k)} &\text{(Jensen's inequality)} \\
    &\leq \frac{\sum_k r_k\,E_B(k)}{\sum_k E_B(k)} \cdot \frac{\sum_k E_B(k)^2}{\sum_k E_B(k)} & \quad \text {(Chebyshev covariance inequality)}\\
    &\leq \cdot \frac{\sum_k E_B(k)^2}{\sum_k E_B(k)} &\text{since} \quad \frac{\sum_k r_k\,E_B(k)}{\sum_k E_B(k)} \leq 1 \\
    &= R(\theta_B)
\end{aligned}
\end{equation*}

where Jensen's Inequality states that $r_k^2 \leq r_k$ for every $k$, given $0 \leq r_k \leq 1$ ; and Chebyshev covariance inequality states that for two sequences $f(k)$ and $g(k)$ that are oppositely monotone (one non-decreasing, the other non-increasing) and non-negative weights $w_k \geq 0$:
\[
    \left(\sum_k w_k\right)\left(\sum_k w_k f(k)\,g(k)\right) \;\leq\; \left(\sum_k w_k f(k)\right)\left(\sum_k w_k g(k)\right),
\]
or equivalently $\mathbb{E}_w[f(k)\,g(k)] \leq \mathbb{E}_w[f(k)]\,\mathbb{E}_w[g(k)]$. Given the probability weights $w_k = E_B(k)/\sum_j E_B(j)$, since $r_k$ is non-decreasing in $k$ and $E_B(k)$ is non-increasing in $k$ by assumption, applying this with $f(k) = r_k$ and $g(k) = E_B(k)$ gives $\mathrm{Cov}_{w}(r,\,E_B) \leq 0$, i.e.,
\[
    \mathbb{E}_{w}[r_k\,E_B(k)] \;\leq\; \mathbb{E}_{w}[r_k]\;\mathbb{E}_{w}[E_B(k)].
\]

Finally, by ~\eqref{eq:bias_error_psd_app}, $\mathcal{B}^{(own)}_\theta$ is increasing in $R$, so $R(\theta_A) \leq R(\theta_B)$ implies $\mathcal{B}^{(own)}_{\theta_A}(\sigma) \leq \mathcal{B}^{(own)}_{\theta_B}(\sigma)$.
\end{proof}

\paragraph{Remark.}
The decreasing-$E_\theta(k)$ condition in (2) is satisfied whenever $P(k)$ is decreasing (e.g.\ turbulence spectra $P(k)\sim k^{-\beta}$), since the MMSE residual of the Wiener denoiser $E(k) = \sigma^2 P(k)/(\bar\alpha P(k)+\sigma^2)$ is an increasing function of $P(k)$.  The spectral-bias condition (the tendency of gradient-trained networks to learn low-frequency components first) is necessary: without it, a uniform rescaling $E_A(k) = c\,E_B(k)$ leaves $R$ unchanged, and hence $\mathcal{B}^{(own)}$ unchanged too.

\newpage

\section{Proof of Proposition~\ref{prop:slow_decrease}}
\label{app:proof_slow_decrease}

\begin{proposition}[\textbf{Slow Error Decrease Principle}, restated]
Under the stability condition established in Proposition~\ref{prop:instability_threshold} and assumptions \textbf{(A1)}--\textbf{(A2)}, the schedule minimizing \eqref{bias-free} contains no unnecessary noise levels: every intermediate $\sigma_k$ is necessary in the sense that removing it — i.e., going directly from $\sigma_{k-1}$ to $\sigma_{k+1}$ — would violate the bias constraint, i.e. $\mathcal{B}^{(2S)}(k-1) > \tau$. In particular, the greedy construction that always takes the largest feasible jump is optimal.
\end{proposition}

The proof relies on the instability threshold (Proposition~\ref{prop:instability_threshold}) and the following assumptions:

\begin{enumerate}
    \item[\textbf{(A1)}] \textbf{Finite capacity trade-off:} Given a model trained on schedule $\mathcal{S}$, adding an extra noise level $\sigma'$ to $\mathcal{S}$ and retraining to convergence necessarily increases $\mathcal{E}^{\text{clean}}_\theta(\sigma_s)$ for some $\sigma_s \in \mathcal{S}$.
    \item[\textbf{(A2)}] \textbf{Monotone error decay:} $\mathcal{E}^{\text{clean}}_\theta(\sigma)$ is strictly decreasing in $\sigma$: lower noise levels yield strictly lower clean-input error.
\end{enumerate}

\begin{proof}
The final inference error decomposes as $\mathcal{E}^{\text{inf}}_\theta(0) = \mathrm{REB}(0) \cdot \mathcal{E}^{\text{clean}}_\theta(0)$. Minimizing this requires both $\mathrm{REB}(0) \approx 1$ and $\mathcal{E}^{\text{clean}}_\theta(0)$ as small as possible.

Suppose the optimal schedule $\mathcal{S}^*$ contains an intermediate noise level $\sigma_k$ ($0 < k < T$) that is unnecessary, i.e., $\mathcal{B}^{(2S)}(\sigma_{k-1}, \sigma_{k+1}) \leq \tau$ so that $\sigma_k$ can be skipped without violating stability. Removing $\sigma_k$ from $\mathcal{S}^*$ yields a strictly shorter schedule. By \textbf{(A1)}, retraining on this shorter schedule frees capacity, which the model can reallocate to $\sigma_0$, strictly decreasing $\mathcal{E}^{\text{clean}}_\theta(\sigma_0)$ — possible by \textbf{(A2)}, since $\sigma_0$ is the lowest noise level and thus has the most room for improvement relative to other levels. This strictly reduces the final inference error, contradicting the optimality of $\mathcal{S}^*$.

Therefore every intermediate step in the optimal schedule is necessary. Note that the bias constraint $\mathcal{B}^{(2S)}(\sigma_t, \sigma_{t+1}) \leq \tau$ need not be tight at each step — it suffices that no larger jump is feasible. Furthermore, satisfying the bias constraint becomes increasingly easy at higher noise levels, since the instability threshold $\gamma(\sigma, \tau)$ is increasing in $\sigma$: the model can afford larger clean-input errors at high noise, allowing larger jumps. This justifies the greedy construction, which always takes the largest feasible jump and is therefore optimal.
\end{proof}

\newpage

\section{Accuracy of the Proxy Estimator}\label{gradient-appendix}

\begin{figure}[H]
    \centering
    \includegraphics[width=\linewidth]{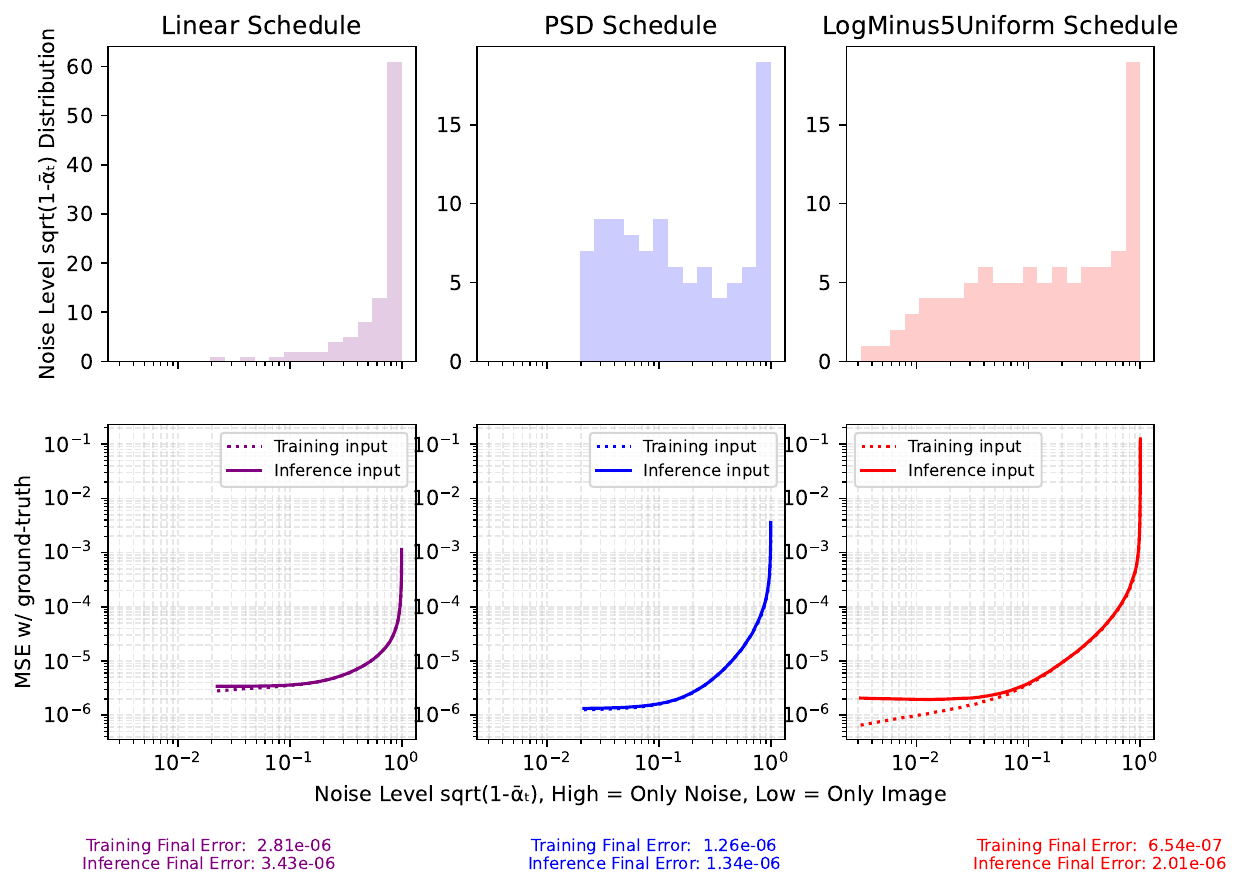}
    \includegraphics[width=\linewidth]{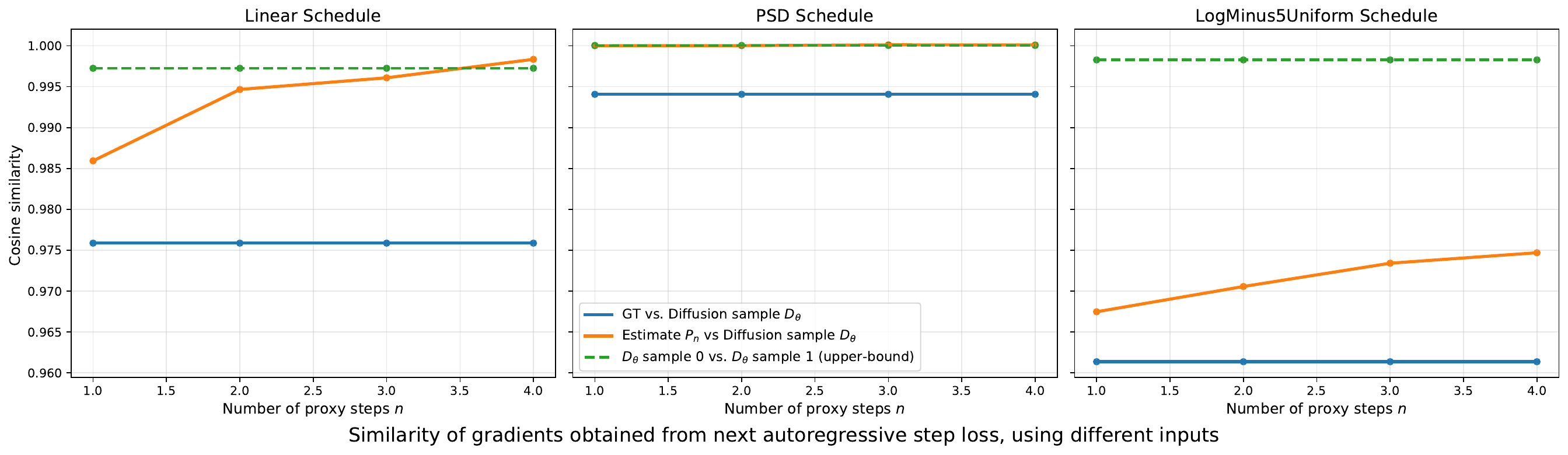}
    \caption{Similarity of gradients obtained from the second step loss $\mathcal{L}_{\text{P-2UT}}$, using as input either the ground-truth data (= teacher forcing), the true previous step diffusion sample (= full unrolled training) or the proxy estimate $\mathcal{P}^{(n)}_\theta$, as a function of the number of steps $n$. PSD and LogMinus5Uniform Schedules are arbitrary schedules that we defined. The alignment of the gradients well reflect the different exposure-biases. In particular, with PSD Schedule, which doesn't suffer from exposure-bias, the proxy estimate is already fully accurate in a single step ($n=1$). For LogMinus5Uniform Schedule, the estimate still isn't accurate after 4 steps.}
    \label{fig:gradient_similarity}
\end{figure}

\newpage
\section{FSD Evolution}\label{FSD-appendix}

\begin{figure}[H]
    \centering
    \includegraphics[width=0.8\linewidth]{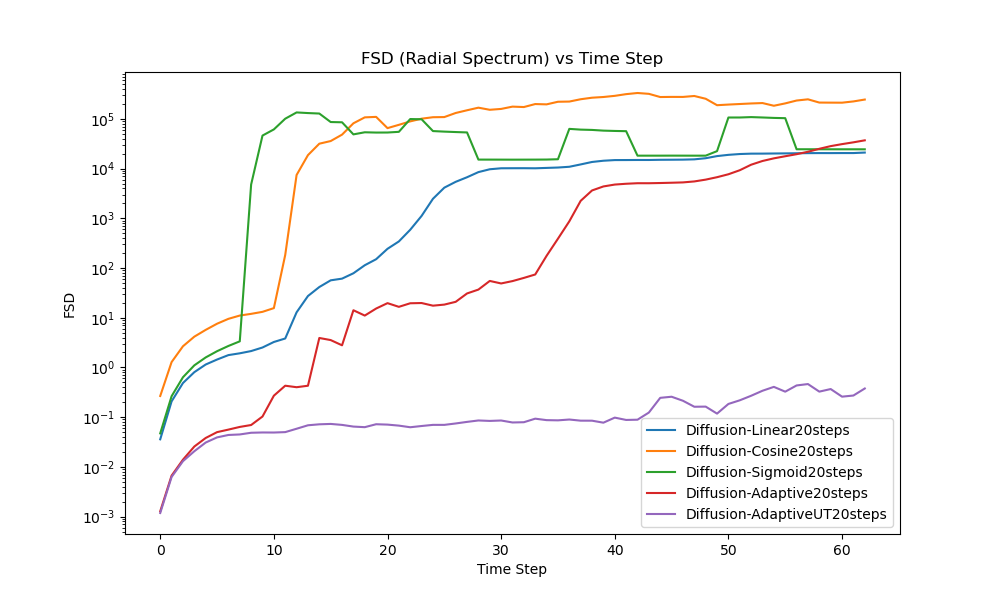}
    \caption{Kolmogorov Flow}
    \label{fig:fsd_kolmo}
\end{figure}

\begin{figure}[H]
    \centering
    \includegraphics[width=0.8\linewidth]{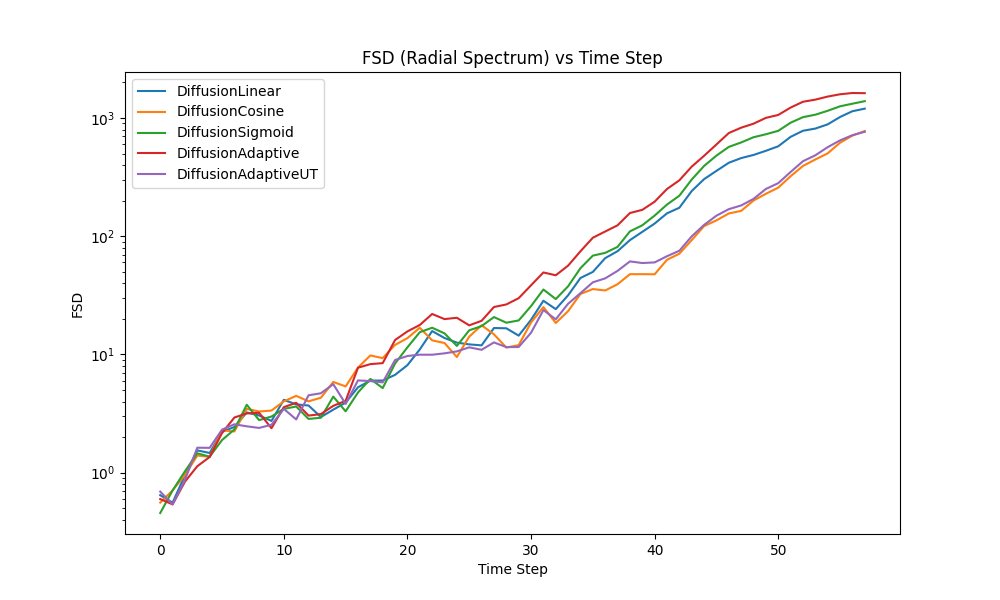}
    \caption{Transonic Flow}
    \label{fig:fsd_tra}
\end{figure}

Our proxy unrolled training is particularly useful when the model suffers from artifacts, as is the case on the Kolmogorov Flow. The improvement is however marginal when the distribution shift is steady and without jumps, as in Transonic Flow.

\newpage 

\section{Implementation details} \label{implementation}

\subsection{Algorithm Pseudocode}

\begin{algorithm}[tbh]
\caption{Phase 1 --- Exploration}
\label{alg:exploration}
\begin{algorithmic}[1]
\STATE {\bfseries Require:} bias tolerance $\tau$, exploration grid $\boldsymbol{\sigma}^{\text{exp}} = \{\sigma^{(1)}, \ldots, \sigma^{(N)}\}$ (log-uniform), shared weights $w = 1$
\STATE {\bfseries Initialize:} active schedule $\boldsymbol{\sigma}^{\text{active}} \leftarrow \boldsymbol{\sigma}^{\text{exp}}$, solved set $\boldsymbol{\sigma}^{\text{solved}} \leftarrow \emptyset$
\WHILE{$\boldsymbol{\sigma}^{\text{active}} \neq \emptyset$}
    \STATE Train $\theta$ on $\boldsymbol{\sigma}^{\text{active}}$ with uniform weights until convergence
    \FOR{each $\sigma^{(i)} \in \boldsymbol{\sigma}^{\text{active}}$}
        \IF{$\mathcal{B}^{(\text{own})}_\theta(\sigma^{(i)}) \leq \tau$}
            \STATE Save checkpoint $\theta^*(\sigma^{(i)}) \leftarrow \theta$
            \STATE $\boldsymbol{\sigma}^{\text{active}} \leftarrow \boldsymbol{\sigma}^{\text{active}} \setminus \{\sigma^{(i)}\}$
            \STATE $\boldsymbol{\sigma}^{\text{solved}} \leftarrow \boldsymbol{\sigma}^{\text{solved}} \cup \{\sigma^{(i)}\}$
        \ENDIF
    \ENDFOR
    \IF{no new level solved this pass}
        \STATE \textbf{break}
    \ENDIF
\ENDWHILE
\STATE \textbf{return} $\{\theta^*(\sigma)\}_{\sigma \in \boldsymbol{\sigma}^{\text{solved}}}$
\end{algorithmic}
\end{algorithm}

\begin{algorithm}[tbh]
\caption{Phase 2 --- Greedy Schedule Construction}
\label{alg:greedy}
\begin{algorithmic}[1]
\STATE {\bfseries Require:} checkpoints $\{\theta^*(\sigma)\}_{\sigma \in \boldsymbol{\sigma}^{\text{solved}}}$, bias tolerance $\tau$
\STATE $\sigma_0 \leftarrow \min \boldsymbol{\sigma}^{\text{solved}}$, \quad $t \leftarrow 0$, \quad $\mathcal{S} \leftarrow [\sigma_0]$
\WHILE{$\sigma_t < \sigma_T$}
    \STATE $\sigma_{t+1} \leftarrow \max \left\{ \sigma' \in \boldsymbol{\sigma}^{\text{solved}} : \mathcal{B}^{(2S)}_{\theta^*(\sigma'),\, \theta^*(\sigma_t)}(\sigma_t, \sigma') \leq \tau \right\}$
    \STATE $\mathcal{S} \leftarrow \mathcal{S} \cup [\sigma_{t+1}]$, \quad $t \leftarrow t + 1$
\ENDWHILE
\STATE Fine-tune a single shared model $\theta$ on $\mathcal{S}$, warm-started from $\theta^*(\sigma_0)$
\STATE \textbf{return} $\mathcal{S}$, $\theta$
\end{algorithmic}
\end{algorithm}

\subsection{Training Hyperparameters}

We set $\tau = 1.05$ across all tasks as it constrains the two-steps bias while allowing for flexibility. A stricter value could potentially reduce the reconstruction error, but at the cost of making the optimization process more complex.

If not mentioned otherwise, each baseline diffusion model uses $T=20$ steps. A training run contains $E=1000$ epochs on Kolmogorov Flow and KS, $E=2000$ epochs on Transonic Flow.

On the other hand, unrolled fine-tuning is done for 200 epochs on Kolmogorov Flow and KS, and 400 epochs for Transonic.

\section{Dataset and Data Generation}

\begin{itemize}
    \item Kolmogorov Flow : We use the generation pipeline from \cite{rozet2023score}, to generate 800 trajectories for training, 100 for testing, 100 for validation. Generation scripts can be obtained in \href{https://github.com/francois-rozet/sda}{SDA}.
    \item Transonic Flow : We obtain the dataset from the ACDM benchmark \cite{kohl2023benchmarking} (repository \href{https://github.com/tum-pbs/autoreg-pde-diffusion}{ACDM}). The testing experiments are ran over the Extrapolate test case.
    \item Kuramoto-Sivashinsky : We generate the data using the pipeline provide in \cite{brandstetter2022lie} (repository \href{https://github.com/brandstetter-johannes/LPSDA/tree/master}{LPSDA}). We increase the number of validation and testing trajectories to 512, and fix the timestep $\Delta t$ to 0.2, as well as the grid spacing $\Delta x$, following \cite{shysheya2024conditional}.
\end{itemize}

\end{document}